\documentclass[twoside,11pt]{article}

\usepackage{jmlr2e}

\usepackage{xcolor}

\usepackage[cmex10]{amsmath}

\usepackage{graphicx}
\graphicspath{{./figures/},{./poster/}}
\DeclareGraphicsExtensions{.pdf,.jpeg,.png}

\usepackage{grffile}

\usepackage{hyperref}
\usepackage{paralist}

\usepackage[ruled,vlined]{algorithm2e}
\usepackage{array}

\newcolumntype{V}{>{\centering\arraybackslash} m{.5\linewidth} }

\newcommand{\children}[0]{\operatorname{children}}

\newcommand{\cost}[0]{\mathtt{c}}
\newcommand{\couplingcost}[0]{\mathtt{cost}}
\newcommand{\coupling}[0]{\pi}

\def\argmin{\mathop{\rm argmin}}

\newcommand{\Xsp}{{\mathbf{X}}}
\newcommand{\Ysp}{{\mathbf{Y}}}

\jmlrheading{18}{2017}{1-32}{3/16; Revised 6/17}{8/17}{16-108}{Samuel Gerber and Mauro Maggioni}

\ShortHeadings{Multiscale Optimal Transport}{Gerber and Maggioni}
\firstpageno{1}

\title{Multiscale Strategies for Computing Optimal Transport}

\author{\name Samuel Gerber$^1$ \email samuel.gerber@kitware.com \\
  \name Mauro Maggioni$^{2,3,4}$ \email mauro.maggioni@jhu.edu \\
       \addr $^1$Kitware, NC, U.S.A\\
       \addr Department of $^2$\!Mathematics, $^3$\!Applied Mathematics, $^4$\!Institute for Data Intensive Engineering and Science, Johns Hopkins University, Baltimore, MD, U.S.A.
}

\editor{Nikos Vlassis}

\begin{document}

\maketitle

\begin{abstract}%
This paper presents a multiscale approach to efficiently compute approximate
optimal transport plans between point sets. It is particularly well-suited for
point sets that are in high-dimensions, but are close to being intrinsically
low-dimensional.  The approach is based on an adaptive multiscale decomposition
of the point sets. The multiscale decomposition yields a sequence of optimal
transport problems, that are solved in a top-to-bottom fashion from the
coarsest to the finest scale.  We provide numerical evidence that this
multiscale approach scales approximately linearly, in time and memory, in the
number of nodes, instead of quadratically or worse for a direct solution.
Empirically, the multiscale approach results in less than one percent relative
error in the objective function.  Furthermore, the multiscale plans constructed
are of interest by themselves as they may be used to introduce novel features
and notions of distances between point sets. An analysis of sets of brain MRI
based on optimal transport distances illustrates the effectiveness of the
proposed method on a real world data set. The application demonstrates that
multiscale optimal transport distances have the potential to improve on
state-of-the-art metrics currently used in computational anatomy.  
\end{abstract}

\section{Introduction}
\label{sec:introduction}

The study of maps between shapes, manifolds and point clouds is of great
interest in a wide variety of applications. There are many data types, e.g.
shapes (modeled as surfaces), images, sounds, and many more, where a similarity
between a pair of data points involves computing a map between the points, and
the similarity is a functional of that map.  The map between a pair of data
points however contains much more information than the similarity measure
alone, and the study of networks of such maps have been successfully used to
organize, extract functional information and abstractions, and help regularize
estimators in large collections of shapes \citep{bigdata5,bigdata10,bigdata8}.
The family of maps to be considered depends on the type of shape, manifold or
point cloud, as well as on the choice of geometric features to be preserved in
a particular application.  These considerations are not restricted to data sets
where each point is naturally a geometric object: high-dimensional data sets of
non-geometric nature, from musical pieces to text documents to trajectories of
high-dimensional stochastic dynamical systems, are often mapped, via feature
sets, to geometric objects. The considerations above therefore apply to a very
wide class of data types.

In this paper we are interested in the problem where each object is a point
cloud -- a set of points in $\mathbb{R}^D$ -- and will develop techniques for
computing maps from one point cloud to another, in particular in the situation
where $D$ is very large, but the point clouds are close to being
low-dimensional, for example they may be samples from a $d$-dimensional smooth
manifold $\mathcal{M}$ ($d\ll D$). The two point clouds may have a different
number of points, and they may arise from a sample of a low-dimensional
manifold $\mathcal{M}$ perturbed by high-dimensional noise (for more general
models see the works by~\citet{LMR:MGM1},~\citet{MMS:NoisyDictionaryLearning}
and~\citet{LiaoMaggioni}).  In this setting we have to be particularly careful
in both the choice of maps and in their estimation since sampling and noise
have the potential to cause significant perturbations.

We find optimal transport maps rather well-suited for these purposes. They
automatically handle the situation where the two point clouds have different
cardinality, they handle in a robust fashion noise, and even changes in
dimensionality, which is typically ill-defined, for point clouds arising from
real-world data~\citep{LMR:MGM1}. Optimal transport has a very long history in
a variety of disciplines and arises naturally in a wide variety of contexts,
from optimization problems in economics and resource allocation, to mathematics
and physics, to computer science (e.g. network flow algorithms).  Thus,
applications of optimal transport range from logistics and
economics~\citep{beckmann:eco1952,carlier:jcc2008}, geophysical
models~\citep{cullen:book2006}, image
analysis~\citep{rubner:iccv1998,haker:ijcv2004} to machine
learning~\citep{cuturi:arxiv2013,cuturi:jmlr2014}.  Despite these widespread
applications, the efficient computation of optimal transport plans remains
challenging, especially in complex geometries and in high dimensions. 

\begin{figure}[!b]
\centering
\begin{tabular}{cc}
\includegraphics[width=0.45\linewidth]{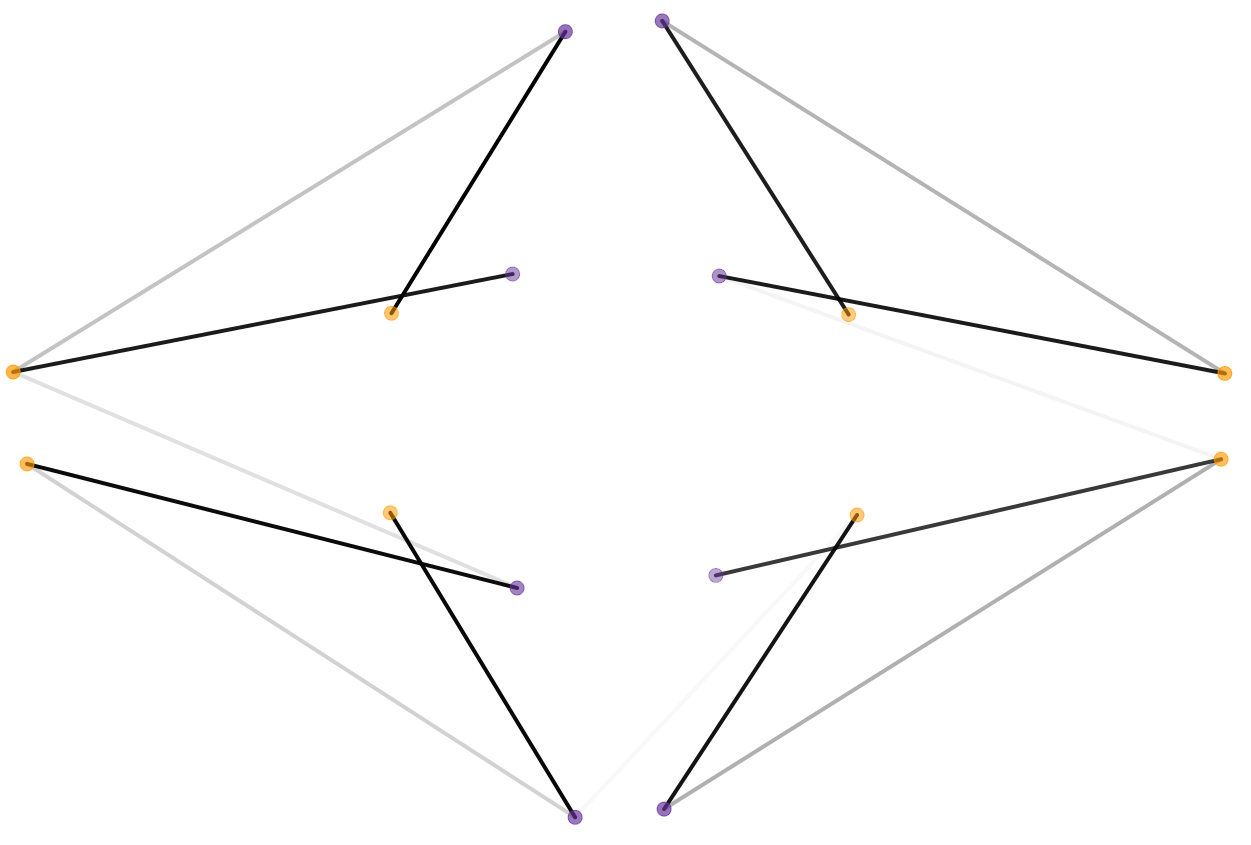} &
\includegraphics[width=0.45\linewidth]{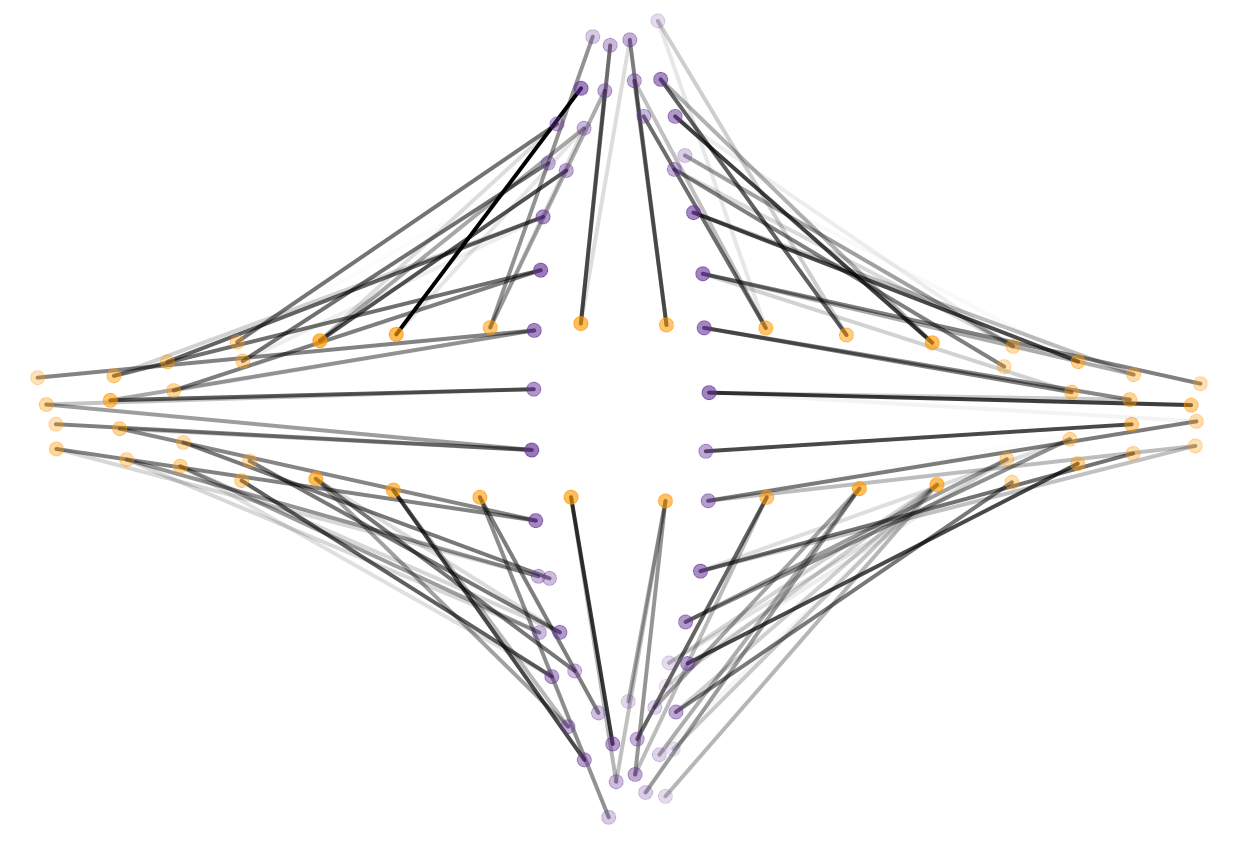} \\
\includegraphics[width=0.45\linewidth]{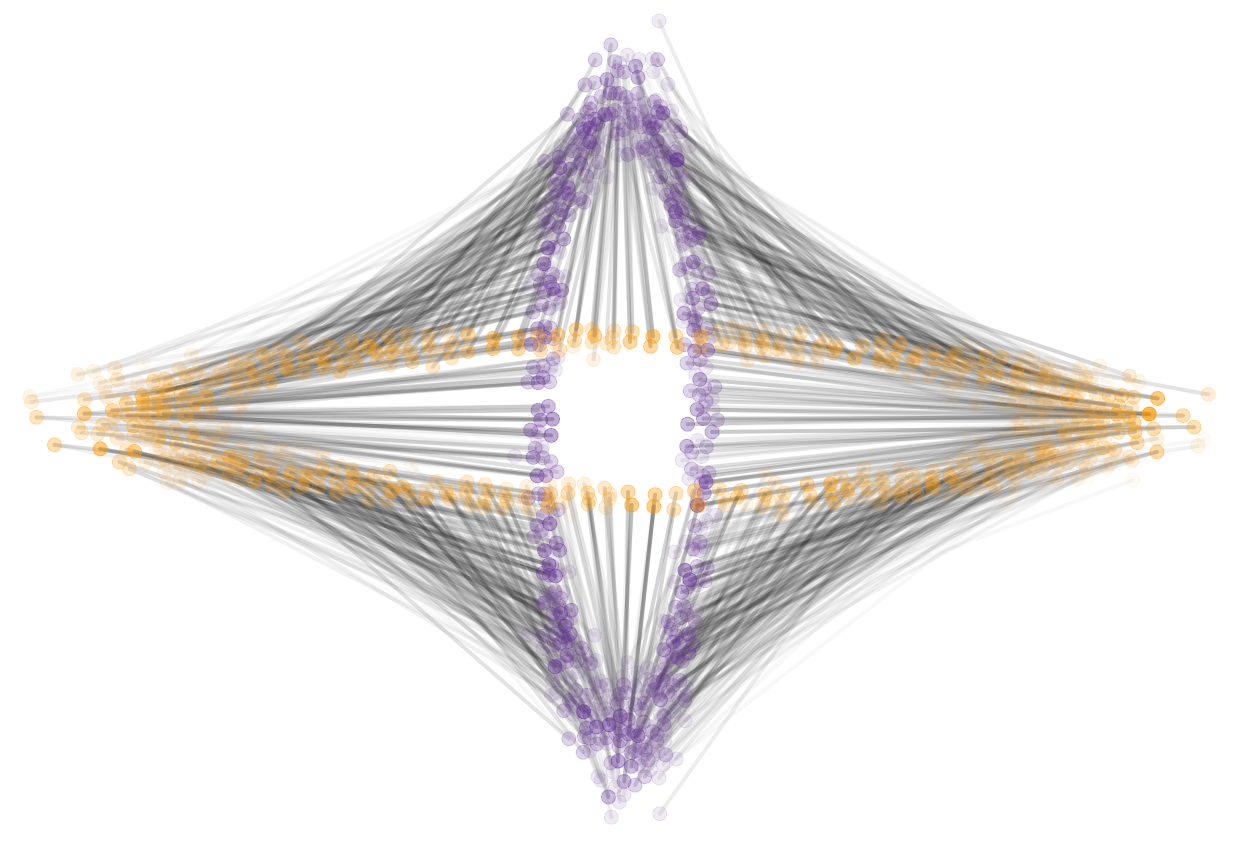} &
\includegraphics[width=0.45\linewidth]{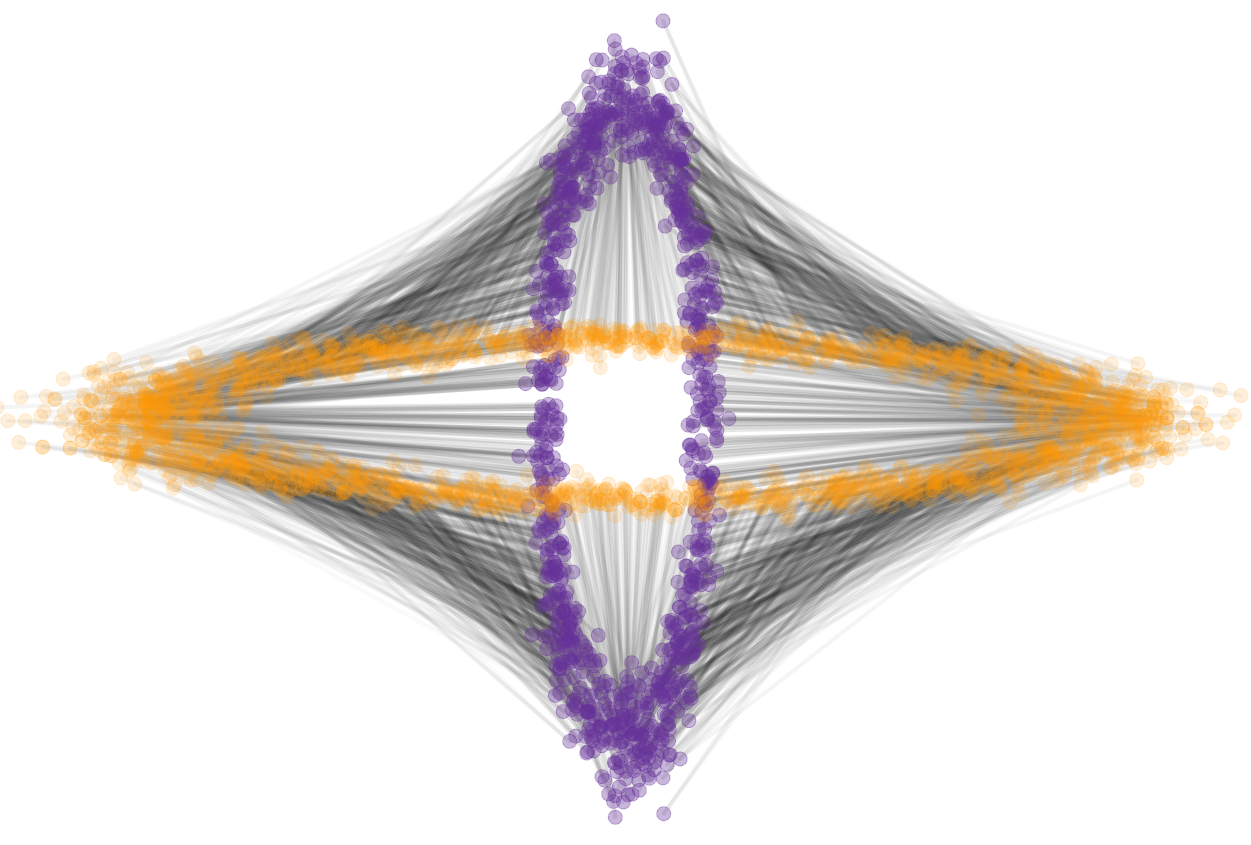} 
\end{tabular}
\caption{\label{fig:multiscale-strategy}
  Optimal transport between two noisy elliptical shapes in $2$-D. We
visualize coarse-to-fine approximations to the source (purple dots) and target
(yellow dots) point clouds, together with the optimal transportation plan at
each scale (gray edges). The intensity of the points is proportional to mass
-- at coarser scales each point represents an agglomerate of points at finer
scales, and its mass is the sum of the masses of the points it represents.
Similarly the intensity of the lines is proportional to the amount of mass
transported along the line. Notice that the optimal transport formulation
permits multiple lines exiting a source point or entering a target point.}
\end{figure}
For point sets the optimal transport problem can be solved by a specialized
linear program, the minimum network flow
problem~\citep{Ahuja:1993:NFT:137406,Tarjan1997}. The minimum network flow
problem has been extensively studied in the operations research community and
several fast algorithms exist. However, these algorithms, at least on desktop
hardware, do not scale beyond a few thousand source and target points.  Our
framework extends the applications of these algorithms to problem instances
several orders of magnitude larger, under suitable assumptions on the geometry
of the data.  We exploit a multiscale representation of the source and target
sets to reduce the number of variables in the linear program and quickly find
good initial solutions, as illustrated in Figure~\ref{fig:multiscale-strategy}.
The optimal transport problem is solved (with existing algorithms) at the
coarsest scale and the solution is propagated to the next scale and refined.
This process is repeated until the finest scale is reached. This strategy,
discussed in detail in Section~\ref{sec:mdop}, is adaptable to memory
limitations and speed versus accuracy trade-offs. For some of the refinement
strategies, it is guaranteed to converge to the optimal solution.

Our approach draws from a varied set of work that is briefly summarized in
Section~\ref{sec:background}. The proposed approach generalizes and builds on
previous and concurrently developed hierarchical
methods~\citep{glimm:arxiv2011,schmitzer2013hierarchical,schmitzer2015sparse,oberman2015efficient}.
The work in this paper adds the following contributions:
\begin{compactitem}
\item The description of a general multiscale framework for discrete optimal
  transport that can be applied in conjunction with a wide range of optimal
  transport algorithms.
\item A set of propagation and refinement heuristics, including approaches
  that are similar and/or refine existing ones
  \citep{glimm:arxiv2011,oberman2015efficient,schmitzer2015sparse} as well as
  novel ones. In particular we propose a novel propagation strategy based on
  capacity restrictions of the network flow problem at each scale. This new
  approach proves to be very efficient and accurate in practice. Overall, the
  heuristics result empirically in a linear increase in computation time with
  respect to data set size.  
\item An implementation in the R package {\em mop} that allows the combination
  of multiple heuristics to tailor speed and accuracy to the requirements of
  particular applications.
\end{compactitem}
Compared to other linear programming based approaches, the multiscale approach results
in a speedup of up to multiple orders of magnitude in large problems and
permits to solve approximately transportation problems of several orders of
magnitudes larger than previously possible. Comparing to PDE based approaches
is difficult and PDE based methods are limited to low-dimensional domains and
specific cost metrics. The proposed framework is demonstrated on several
numerical examples and compared to the state-of-the-art approximation algorithm
by~\citet{cuturi:nips2013}.

\section{Background}
\label{sec:background}
Optimal transport is the problem of minimizing the cost of moving a source
probability distribution to a target probability distribution given a function
that assigns costs to moving mass from source to target locations. The
classical formulation by~\citet{monge:book1781} considers the minimization over
mass preserving mappings.  Later~\citet{kantorovitch:ms1958} considered the
same problem but as a minimization over couplings between the two probability
measures, which permits to split mass from a single source across multiple
target locations.  More precisely, for two probability measures  $\mu$ and
$\nu$ on probability spaces ${\Xsp}$ and ${\Ysp}$ respectively, a coupling of
$\mu$ and $\nu$ is a measure $\coupling$ on ${\Xsp}\times{\Ysp}$ such that the
marginals of $\coupling$ are $\mu$ and $\nu$, i.e.
$\coupling(A\times{\Ysp})=\mu(A)$ for all ($\mu$-measurable) sets
$A\subseteq{\Xsp}$ and $\coupling({\Xsp}\times B)=\nu(B)$ for all
($\nu$-measurable) sets $B\subseteq{\Ysp}$. We denote by $\mathcal{C}(\mu,\nu)$
the set of couplings between $\mu$ and $\nu$.  Informally, we may think of
$d\coupling(x,y)$ as the amount of infinitesimal mass to be transported from
source $x$ to destination $y$, with the condition that $\coupling$ is a
coupling guaranteeing that the source mass is distributed according to $\mu$
and the destination mass is distributed according to $\nu$.  Such couplings
always exist: we always have the trivial coupling $\coupling=\mu\times\nu$. The
trivial coupling is uninformative, every source mass $d\mu(x)$ is transported
to the same target distribution $\nu$. In Monge's formulation the coupling is
restricted to the special form $\coupling(x,y)=\delta_{T(x)}$ where $\delta_y$
is the Dirac-$\delta$ measure with mass at $y\in{\Ysp}$, and $T$ is a function
$\Xsp\rightarrow\Ysp$: in this case the coupling is ``maximally informative''
in the sense that there is a function mapping each source $x$ to a single
destination $y$; in particular the mass $d\mu(x)$ at $x$ is not split into
multiple portions that are shipped to different target $y$'s.  

To define optimal transport and optimal couplings, we need a cost function
$\cost(x,y)$ on ${\Xsp}\times{\Ysp}$ representing the work or cost needed to
move a unit of mass from $x$ to $y$. Then for every coupling $\coupling$ we may
define the cost of $\coupling$ to be 
\begin{equation}
\mathbb{E}_{(X,Y)\sim\coupling}[\cost(X,Y)]\,,
\end{equation}
with $(X,Y)$ being a pair of random variables distributed according to $\coupling$.
An optimal coupling $\coupling$ minimizes this cost over all choices of couplings.
When seeking an optimal transportation plan, the above becomes 
$$
\mathbb{E}_{X\sim\mu}[\cost(X,T(X))]=\int_{{\Xsp}} \cost(x,T(x)) d\mu(x)\,,
$$ 
and one minimizes over all (measurable) functions $T:{\Xsp}\rightarrow{\Ysp}$.
One often designs the cost function $c$ in an application-dependent fashion,
and the above framework is extremely general.  When ${\Xsp}={\Ysp}$ is a metric
space with respect to a distance $\rho$ (with suitable technical conditions
relating the metrics and the measures $\mu$, $\nu$ that we will not delve into,
referring the reader to \cite{villani:book2009}), then distinguished choices
for the cost are those that are related to the metric structure. The natural
choice of $\cost(x,y)=\rho(x,y)^p$, for some $p>0$ leads to the definition of
the Wasserstein-Kantorovich-Rubinstein metric on the space of probability
measures on ${\Xsp}$: 
\begin{equation}
W_p(\mu,\nu):=\min_{\coupling\in\mathcal{C}(\mu,\nu)}\left(\int_{{\Xsp}}\int_{{\Xsp}}
\rho(x,y)^p d\coupling(x,y)\right)^{\frac1p}\,.  
\end{equation}

Computational solutions to optimal transport split roughly in two settings:
Approaches based on the solution of partial differential equations derived from
the continuous optimal transport formulation, briefly discussed in
Section~\ref{sec:continuous} and, more relevant to this paper, combinatorial
optimization methods to directly solve for a discrete optimal transport plan
discussed in Section~\ref{sec:lp}.

\subsection{Continuous Optimal Transport}
\label{sec:continuous}
In case that at least the source distribution admits a density, and when the
cost function is the squared Euclidean distance, the optimal coupling is
deterministic, i.e. there exists a transport map, and the optimal solution is
the gradient of a convex function~\citep{brenier:cpam1991}. This has been
exploited to solve the optimal transport problem by numerical partial
differential equation
approaches~\citep{benamou:nm2000,angenent:jma2003,haker:ijcv2004,iollo:jcp2011,papadakis:arxiv2013,benamou2014numerical}.

An alternative formulation proposed by~\citet{aurenhammer:algorithmica1998}
shows that the optimal transport from a source density to a target distribution
of a set of weighted Dirac delta's can be solved through a finite dimensional
unconstrained convex optimization. \citet{merigot:cgf2011} proposes a
multiscale approach for the formulation of~\citet{aurenhammer:algorithmica1998}.

Both the numerical PDE based approaches as well as the unconstrained convex
optimization require a discretization of the full domain which is generally not
feasible for higher dimensional domains. For arbitrary cost functions and
distributions the optimal transport problem does typically not result in a
deterministic coupling and can not be solved through PDE based approaches.

\subsection{Discrete Optimal Transport and Linear Programming}
\label{sec:lp}

For two discrete distributions $\mu = \sum_1^n w(x_i) \delta(x_i)$ and $ \nu =
\sum_1^m v(y_i) \delta(y_i)$ with $\sum w(x_i) = \sum v(y_i) = 1$ the optimal
transport problem is equivalent to the linear program 
\begin{equation}
\min_\coupling \sum_{\substack{i=1,\dots,n\\ j=1,\dots,m}} 
      \cost(x_i, y_j) \coupling(x_i, y_j) \quad \text{s.t.}\quad 
\begin{cases}
\sum_j \coupling(x_i, y_j) = \mu(\{x_i\})= w(x_i) & \\ 
\sum_i \coupling(x_i, y_j) = \nu(\{y_j\}) = v(y_j) & \\
 \coupling(x_i,y_j)\ge 0
\end{cases}\,.
\label{e:LPformulation}
\end{equation} The solution is called an optimal coupling $\coupling^*$, and the
minimum value attained at $\coupling^*$, is called the cost of $\pi^*$, or the
optimal cost of the transport problem, and is denoted by
$\couplingcost(\coupling^*)=\sum_{i,j} \cost(x_i,y_j)\pi^*(x_i,y_j)$. The
constraints enforce that $\coupling$ is a coupling. The variables
$\coupling(x_i, y_j)$ correspond to the amount of mass transported from
source $x_i$ to target $y_j$, at cost $\cost(x_i, y_j)$.  The linear
constraints are of rank $n+m-1$: when $n+m-1$ of the constraints are satisfied,
either the $n$ constraints of the source density or the $m$ constraints of the
target density are satisfied. Since the sum of outgoing mass is equal to the sum
of incoming mass, i.e. $\mu(\Xsp) = \nu(\Ysp)$, it follows that all constraints
must be satisfied. The optimal solution lies, barring degeneracies, on a corner
of the polytope defined by the constraints, i.e. is a basic feasible solution. 
This implies that exactly $n+m-1$ entries of the optimal coupling $\coupling$
are non-zero, i.e. $\pi$ is a sparse matrix.  The optimal coupling is a Monge
transport map if and only if all the mass of a source $x_i$ is transported to
exactly one target location. A Monge solution does not exist for every optimal
transport problem, in fact a small perturbation of $\mu$ will always suffice to
make a Monge solution impossible. 


For source and target data sets with the same cardinality and equally weighted
$\delta$-functions, Kantorovich's optimal transport problem reduces to
an assignment problem, whose solution is a Monge optimal transport map. In this
special case, the optimal transport problem can be efficiently solved by the
Hungarian algorithm~\citep{kuhn:nrlq1955}. The assignment problem results in
a degenerate linear program since only $n$ entries are non-zero (instead of
$2n-1$). We can therefore think of optimal transport as a robust version of
assignments. This can also be seen from the point of view of convexity: in the
assignment problem, $\coupling$ is a permutation matrix deciding to which $y_j$
each $x_i$ is transported. The convex hull of permutation matrices is exactly
the set of doubly-stochastic matrices, to which a coupling $\pi$ belongs as a
consequence of the constraints in \eqref{e:LPformulation}.

For point sets with different cardinalities and/or points with different masses
the optimal transport problem can be solved by a linear program and is a
special case of the minimum cost network flow problem. The minimum cost flow
problem is well studied and a number of
algorithms~\citep{ford:ms1956,klein:ms1967,
cunningham:mp1976,goldberg:stoc1987,bertsekas:or1988,orlin:mp1997} exist for
its solution. This discrete solution approach is not constrained to specific
cost functions and can work with arbitrary cost functions.

However, the linear programming approach neglects possibly useful geometric
properties of the measures and the cost function.  Our work makes
assumptions about the underlying geometry of the measure spaces and the
associated cost function, and in this way is a mixing of the low-dimensional
``geometric PDE'' approaches with the discrete non-geometric optimization
approaches. It exploits the geometric assumptions to relieve the shortcomings
of either approach, namely it scales to high-dimensional data, provided that
the intrinsic dimension is low in a suitable sense, and does not require a mesh
data structure. At the same time we use the geometry of the data
to speed up the linear program, which per--se does not leverage geometric
structures.  

The refinement strategies of the proposed multiscale approach add subsets of paths among
all pairwise paths at each subsequent scale to improve the optimal transport
plan. This strategy of adding paths, is akin to column generation
approaches~\citep{desrosiers:book2005}.  Column generation, first developed
by~\citet{dantzig:informs1960} and~\citet{ford:ms1956}, reduces the number of
variables in a linear program by solving a smaller linear program on a subset
of the original variables and introducing new variables on demand. However, the
proposed approach exploits the geometry of the problem instead of relying on an
auxiliary linear program~\citep{dantzig:informs1960} or shortest path
computations~\citep{ford:ms1956} to detect the entering variables.

\subsection{Approximation Strategies}
In the computer vision literature, the Earth Movers distance or equivalently the
Wasserstein-1 distance, which is simply the cost of the optimal coupling, is
a successful similarity measure for image retrieval~\citep{rubner:iccv1998}. In
this application the transport plan is not of interest but only the final
transport cost.  For this
purpose~\citet{indyk2003fast},~\citet{shirdhonkar:cvpr2008} and
~\citet{andoni:soda2008} developed algorithms that compute an approximate cost
but do not yield a transport plan.  Some of these approaches are based on the
dual formulation of optimal transport, which involves testing against Lipschitz
functions, and observing that Lipschitz functions may be characterized by decay
properties of their wavelet coefficients. In this sense these approaches
are multiscale as well.

To speed up computations in machine learning
applications~\citet{cuturi:nips2013} proposes to {\em smooth} transport plans by
adding a maximum entropy penalty to the  optimal transport formulation. The
resulting optimization problem is efficiently solved through matrix scaling
with Sinkhorn fixed-point iterations.  Because of the added regularization
term, the solution will in general be different from the optimal
transportation. It may however be the case that these particular (or
perhaps other) regularized solutions are better suited for certain
applications.

\subsection{Related Work}
Very recently a number of approaches have been proposed to solve the optimal
transport in a multiscale
fashion~\citep{glimm:arxiv2011,schmitzer2013hierarchical,schmitzer2015sparse,oberman2015efficient}.
\citet{glimm:arxiv2011} design an iterative scheme to solve a discrete optimal
transport problem in reflector design and propose a heuristic for the iterative
refinements based on linear programming duality. This iterative scheme can be
interpreted as a multiscale decomposition of the transport problem based on
geometry of the source and target sets. The proposed potential refinement
strategy extends the heuristic proposed by~\citet{glimm:arxiv2011} to guarantee
optimal solutions and adds a more efficient computation strategy:  Their
approach requires to check all possible variables at the next scale. In
Section~\ref{sec:potential} we introduce a variation of the approach
by~\citet{glimm:arxiv2011} by adding a branch and bound strategy to avoid
checking all variables, and an iterative procedure that guarantees optimal
solutions.  

\citet{schmitzer2013hierarchical} propose a multiscale approach on grids that
uses a refinement strategy based on spatial neighborhoods, akin to the
neighborhood refinement described in Section~\ref{sec:neighborhood}.
\citet{schmitzer2015sparse} uses a multiscale approach to develop a modified
auction algorithm with guaranteed worst case complexity and optimal solutions.
We use data structures that enable us to quickly construct neighborhoods, even
for point clouds that live in high-dimensions, but have low-intrinsic dimension;
we also exploit these structures to not compute all the pairwise possible
costs, as well as the candidate neighborhoods in our propagation steps. This
can result in substantial savings in the scaling of the algorithm, from
$|\mathbf{X}|^3$ to just $|\mathbf{X}|\log|\mathbf{X}|$. We are therefore able
to scale to larger problem, solving on a laptop problems an order of magnitude
larger than those in \citet{oberman2015efficient} for example, and showing
linear scaling on a large range of scales.  \citet{schmitzer2015sparse} use the
c--cyclical monotonicity property of optimal transport
plans~\citep[Chapter~5]{villani:book2009} to construct ``shielding
neighborhoods'' that permit to exclude paths from further consideration. The
idea of shielding neighborhoods is combined with a multiscale strategy that
permits to quickly refine initial neighborhood estimates. 

Finally, in our work we emphasize that the multiscale construction is not only
motivated by its computational advantage, but also as a way of revealing
possibly important features of the optimal transportation map. As we show in
Section \ref{s:brainMR}, features collected from the multiscale optimal
transportation maps leads to improved predictors for brain conditions. More
generally, we expect multiscale properties of optimal transportation maps to be
useful in a variety of learning tasks; the connections between learning and
optimal transportation are still a very open field, to be explored and
exploited.

\section{Multiscale Optimal Transport}
\label{sec:mdop}
Solving the optimal transport problem for two point sets $\Xsp$ and $\Ysp$
directly requires $|\Xsp||\Ysp|$ variables, or paths to consider. In other
words, the number of paths along which mass can be transported grows
quadratically in the number of points and quickly yields exceedingly large
problems.  The basic premise of the multiscale strategy is to solve a sequence
of transport problems based on increasingly accurate approximations of the
source and target point set. The  multiscale strategy helps to reduce the
problem size at each scale by using the solution from the previous scale to
inform which paths to include in the optimization at the next finer scale.
Additionally, the solution at the previous scale helps to find a good
initialization for the current scale which results in fewer iterations to solve
the reduced size linear program.

The multiscale algorithm (see Algorithm~\ref{f:algo} and
Figure~\ref{fig:multiscalepic} for a visual illustration) comprises of three
key elements: 
\begin{itemize}
\item[(I)] A way of {\textit{\textbf{coarsening} the sets}} of source points ${\Xsp}$ and  measure $\mu$ in a multiscale fashion, yielding a chain
\begin{equation}
({\Xsp},\mu)=:({\Xsp}_J,\mu_J) \rightarrow({\Xsp}_{J-1},\mu_{J-1})
\rightarrow\dots\rightarrow({\Xsp}_j,\mu_j)
\rightarrow\dots\rightarrow({\Xsp}_0,\mu_0) 
\end{equation}
connecting the scales from fine to coarse, with ${\Xsp}_j$ of decreasing
cardinality as the scale decreases, and the discrete measure $\mu_j$
``representing'' a coarsification of $\mu$ at scale $j$, with
$\mathrm{supp}(\mu_j)={\Xsp}_j$ (the support of $\mu_j$ is the set of points
with positive measure). Similarly for $\Ysp$ and $\nu$ we obtain the chain
\begin{equation}
({\Ysp},\mu)=:({\Ysp}_J,\nu_J) \rightarrow({\Ysp}_{J-1},\nu_{J-1})
\rightarrow\dots\rightarrow({\Ysp}_j,\nu_j)
\rightarrow\dots\rightarrow({\Ysp}_0,\nu_0) 
\end{equation}
This coarsening step is described in Section~\ref{sec:coarsening} and the
resulting multiscale family of transport problems is discussed in
Section~\ref{sec:multiscaleproblems}.
\item[(II)] A way of {\textit{\textbf{propagating} a coupling}} $\coupling_j$
  solving the transport problem $\mu_j\rightarrowtail\nu_j$ at scale $j$ to a
  coupling $\coupling_{j+1}$ at
  scale $j+1$. This is described in Section~\ref{sec:propagation}.
\item[(III)] A way of {\textit{\textbf{refining} the propagated solution}} to
the optimal coupling at scale $j$. This is described in
Section~\ref{sec:refinement}.  
\end{itemize}


\begin{algorithm}[h]
 \KwIn{Two discrete measures $\mu$ and $\nu$, and a stopping scale $J_0\le J$}
 \KwOut{Multiscale family of transport plans $(\coupling_j: \mu_j\rightarrowtail
        \nu_j)_{j=0}^{J_0}$}

Construct multiscale structures $\{\{{\Xsp}_j,\mu_j)\}_{j=0}^J$ and
$\{\{{\Ysp}_j,\nu_j)\}_{j=0}^J$. 

Let $\tilde \coupling_0$ be an arbitrary coupling $\mu_0\rightarrowtail\nu_0$.

 \For{ $j=0\dots J_0-1$ }{
   Refine initial guess $\tilde\coupling_j$ to the optimal coupling $\coupling_j: \mu_j\rightarrowtail\nu_j$.
      
   Propagate $\coupling_j$ from scale $j$ to scale $j+1$, obtaining a
   coupling $\tilde\coupling_{j+1}$}
\caption{Multiscale Discrete Optimal Transport  
\label{f:algo}
}
\end{algorithm}


\subsection{The Coarsening Step: Multiscale Approximations to $\Xsp$, $\mu$ and $\cost$}
\label{sec:coarsening}
To derive approximation bounds for the error of the multiscale transport
problem at each scale we rely on the notion of a regular family of multiscale
partitions formally described in Section~\ref{sec:mspartitions}. The multiscale
partition is used to define approximations to $\mu$ and $\cost$ at
all scales. An integral part of the definitions is that the constructions can be
interpreted as a tree, with all nodes at a fixed height corresponding to one scale
of the multiscale partitioning. 

We start with some notation needed for the definition of the
multiscale partitions.  Let $({\Xsp},\rho,\mu)$ be a measure metric space with
metric $\rho$ and finite measure $\mu$. Without loss of generality assume that
$\mu(\Xsp)=1$. The metric ball of center $z$ and radius $r$ is $B_z(r)=\{
x\in{\Xsp} : \rho(x,z)<r\}$. We say that ${\Xsp}$ has doubling dimension $d$ if
every ball $B_z(r)$ can be covered by at most $2^d$ balls of radius
$r/2$~\citep{assouad1983plongements}.  Furthermore, a space has a doubling
measure if $\mu(B_z(r))\asymp r^d$, i.e. if there exist a constant $c_1$ such
that for every $z\in\Xsp$ and $r>0$ we have $c_1^{-1}r^d\le\mu(B_z(r))\le
c_1r^d$.  Here and in what follows, we say that $f\asymp g$ if there are two
constants $c_1,c_2>0$ such that for every $z$ in the domain of both functions
$f,g$ we have $c_1 f(z)\le g(z)\le c_2 f(z)$ (and therefore a similar set of
inequalities holds with the roles of $f$ and $g$ swapped), and we say that $f$
and $g$ have the same order of magnitude.  Having a doubling measure implies
having a doubling metric, and up to changing the metric to an equivalent one,
one may choose the same $d$ in the doubling condition for the metric and in
that for the measure: we assume this has been done from now on.  This family of
spaces is rather general, it includes regular domains in $\mathbb{R}^D$, as
well as smooth compact manifolds $\mathcal{M}$ endowed with volume measures.

\subsubsection{Multiscale Approximations to $\Xsp$ and $\mu$}
\label{sec:mspartitions}
A regular family of multiscale partitions, with scaling parameter $\theta>1$,
is a family of sets $\{\{C_{j,k}\}_{k=1}^{K_j}\}_{j=0}^J$, where $j$ denotes
the scale and $k$ indexes the sets at scale $j$, such that: 
\begin{itemize}
\item[(i)] the sets $\{C_{j,k}\}_{k=1}^{K_j}$ form a partition of ${\Xsp}$,
  i.e.  they are disjoint and $\cup_{k=1}^{K_j}C_{j,k} = {\Xsp}$;
\item[(ii)] either $C_{j+1,k'}$ does not intersect a $C_{j,k}$, or it is
  completely contained in it;
\item[(iii)] there exists a constant $A>0$ such that for all $j,k$ we have
  the diameter $\mathrm{diam}(C_{j,k})\le A\theta^{-j}$;
\item[(iv)] each $C_{j,k}$ contains a ``center'' point $c_{j,k}$ such that
  $B_{c_{j,k}}(\theta^{-j})\subseteq C_{j,k}$.
\end{itemize}
To ease the notation we will fix $\theta=2$ in what follows, but the
constructions and results hold, mutatis mutandis, for general $\theta>1$.  The
properties (i) and (ii) above imply that there exists a tree $\mathcal{T}$, with nodes at
scale $j$ (i.e. at distance $j$ from the root) in bijection with
$\{C_{j,k}\}_{k=1}^{K_j}$, such that $(j+1,k')$ is a child of $(j,k)$ if and
only if $C_{j+1,k'}$ is contained in $C_{j,k}$.  Moreover properties (iii) and
(iv), together with the properties of spaces with a doubling measure, imply
that $\mu(C_{j,k})\asymp 2^{-jd}$ and $K_j\asymp 2^{jd}$.  These partitions are
classical in harmonic analysis, mimicking dyadic cubes in Euclidean space, and
they have recently been used to construct multiscale decompositions of data
sets in high-dimensions ~\citep{LMR:MGM1,CM:MGM2,IM:GMRA_CS,6410789}.  We say
that $C_{j+1,k'}$, or even $(j+1,k')$, is a child of $C_{j,k}$ (respectively,
of $(j,k)$) if $C_{j+1,k'}\subseteq C_{j,k}$, and that such $C_{j,k}$
(respectively $(j,k)$) is a parent of $C_{j+1,k'}$ (resp.  $(j+1,k')$).

Given two discrete sets ${\Xsp}$ and ${\Ysp}$ in a doubling metric space of
homogeneous type with dimension $d$, we construct the corresponding families of
multiscale partitions $\{\{C^{{\Xsp}}_{j,k}\}_{k=1}^{K^{{\Xsp}}_j}\}_{j=0}^J$
and $\{\{C^{{\Ysp}}_{j,k}\}_{k=1}^{K^{{\Ysp}}_j}\}_{j=0}^J$ (we assume the same
range of scales to keep the notation simple). The construction of the
multiscale approximations is illustrated in Figure~\ref{fig:multiscalepic}.
\begin{figure}[thb]
\centering
\includegraphics[width=0.99\linewidth]{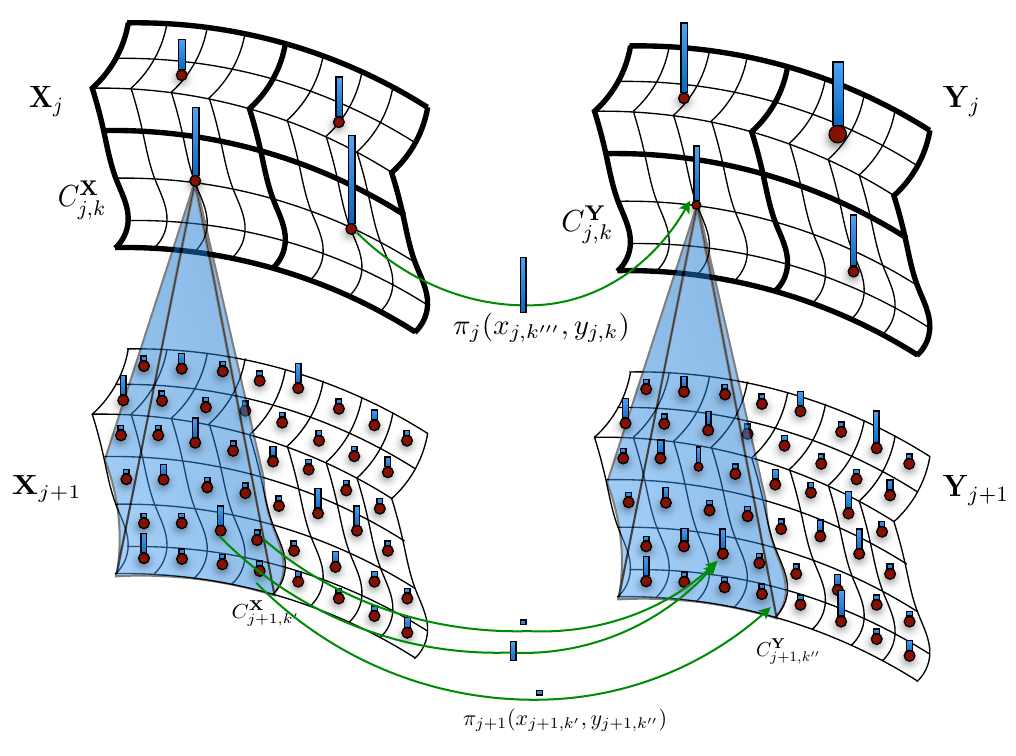}
\caption{
\label{fig:multiscalepic}
An illustration of the multiscale framework. The coarsening step constructs
from a set of weighted points $\Xsp_{j+1}$ at scale $j+1$ a smaller set of
weighted points at scale $j$ (blue bars). An equivalent
coarsening is performed on the target point set $\Ysp_{j+1}$ Neighboring points
at scale $j+1$ are combined into a single representative at scale $j$ with mass
equivalent to the weights of the combined points. The optimal transport plan
(green arrows) is solved at the coarser scale $j$, then propagated to scale
$j+1$ and refined.  
} 
\end{figure}

The space ${\Xsp}_j$ will be the partition at scale $j$, namely the set
$\{C^{{\Xsp}}_{j,k}\}_{k=1}^{K^{{\Xsp}}_j}$, of cardinality $K^{{\Xsp}}_j$.
The measures $\mu$ and $\nu$ may be coarsened in the natural way, by letting
$\mu_j$ be defined recursively on ${\Xsp}_j$ by 
\begin{equation}
 \mu_j( C_{j,k}^{{\Xsp}} )= \sum_{(j+1,k')\text{ child of } (j,k)} \mu_{j+1}( C_{j+1,k'}^{{\Xsp}} )\,,
\label{e:muj}
\end{equation}
and similarly for $\nu_j$.
These are in fact projections of these measures, and may also be interpreted as
conditional expectations with respect to the $\sigma$-algebra generated by the
multiscale partitions.  We can associate a point $\overline c_{j,k}^{{\Xsp}}$
to each $C_{j,k}^{{\Xsp}}$ in various ways, by ``averaging'' the points in
$C_{j+1,k'}^{{\Xsp}}$, for $(j+1,k')$ a child of $(j,k)$. At the finest scale
we may let $\overline{c}^{{\Xsp}}_{J,k}={c}^{{\Xsp}}_{J,k}$ (this being a
``center'' for $C_{J,k}$ as in item (iv) in the definition of multiscale
partitions), and then recursively we defined the coarser centers step from
scale $j+1$ to scale $j$ in one of the following ways: 
\begin{itemize}
\item[(i)] If the metric space is also a vector space a natural definition of
$x_{j,k}:=\overline c_{j,k}^{{\Xsp}}$ is a weighted average of the $\overline
c_{j+1,k'}^{{\Xsp}}$ corresponding to children: 
$$
\overline{c}^{{\Xsp}}_{j,k} = 
 \sum_{(j+1,k')\text{ child of } (j,k)} \mu_{j+1}(\{\overline c_{j+1,k'}^{{\Xsp}}\}) \overline c_{j+1,k'}^{{\Xsp}}\,.
$$
\item[(ii)] In general we can define $x_{j,k}:=\overline c_{j,k}^{{\Xsp}}$ as the point
$$
\overline{c}^{{\Xsp}}_{j,k}=\mathrm{argmin}_{c \in{\Xsp}}\,\,
\sum_{(j+1,k')\text{ child of } (j,k) } \rho^p( c, \overline c_{j+1,k'}^{{\Xsp}} ) \,,
  $$
for some $p \geq 1$, typically $p=1$ (median) or $p=2$ (Fr\'echet mean).
  

\end{itemize} 
Of course similar constructions apply to the space ${\Ysp}_j$, yielding points
$y_{j,k}:=\overline c_{j,k}^{{\Ysp}}$.
We discuss algorithms for these constructions in Section \ref{s:multiscalepointsets}.

\subsubsection{Coarsening the cost function $\cost$}
The multiscale partition provides several ways to coarsen the cost function: for
every $x_{j,k}$ and $y_{j,k'}$ we consider
\begin{itemize}
\item[($\cost$-i)] the pointwise value 
  \begin{equation}
\cost_j(c_{j,k}^{\Xsp}, c_{j,k'}^{\Ysp}) := 
  \cost(x_{j,k}^{\Xsp},y_{j,k'}^{\Ysp})\,,
  \label{e:costpointwise}
\end{equation}
where $x_{j,k}$ and $y_{j,k'}$ are defined in any of the ways above;
\item[($\cost$-ii)] the local average 
\begin{equation}
\cost_j(c_{j,k}^{\Xsp}, c_{j,k'}^{\Ysp}) := 
\argmin_{\alpha} \sum_{\substack{x\in C^{{\Xsp}}_{j,k},\,y \in
C^{{\Ysp}}_{j,k'}}}\left(\alpha-\cost(x,y)\right)^2 
= \frac{\sum_{\substack{x\in C^{{\Xsp}}_{j,k},\,y\in
C^{{\Ysp}}_{j,k'}}}\cost(x,y)}{|C_{j,k}^{{\Xsp}}| |C_{j,k'}^{{\Ysp}}|}\,;
\label{e:costlocalave}
\end{equation}
\item[($\cost$-iii)] the local weighted average 
\begin{align*}
\cost_j(c_{j,k}^{\Xsp},c_{j,k'}^{\Ysp})
&:=\argmin_{\alpha}\sum_{\substack{x\in C^{{\Xsp}}_{j,k},\, y\in
C^{{\Ysp}}_{j,k'}}}\left(\alpha-\cost(x,y)\right)^2
\coupling^*_{j-1}(x_{j-1,k_1},y_{j-1,k'_1})\\ 
&=\frac{\sum_{\substack{x\in C^{{\Xsp}}_{j,k},\,y\in C^{{\Ysp}}_{j,k'}}}\cost(x,y)\coupling^*_{j-1}(x_{j-1,k_1},y_{j-1,k'_1})}{
\sum_{\substack{x\in C^{{\Xsp}}_{j,k},\,y\in C^{{\Ysp}}_{j,k'}}}\coupling^*_{j-1}(x_{j-1,k_1},y_{j-1,k'_1})
}\,,
\end{align*}
where $\pi^*_{j-1}$ is the optimal or approximate transportation plan at scale
$j-1$, defined in \eqref{e:LPformulation_j}; $k_1$ is the unique index for
which $C^{{\Xsp}}_{j,k}\subseteq C^{{\Xsp}}_{j-1,k_1}$ and $k'_1$ is the unique
index for which $C^{{\Ysp}}_{j,k}\subseteq C^{{\Ysp}}_{j-1,k'_1}\,.$ 
\end{itemize}

\subsection{Multiscale Family of Optimal Transport Problems}
\label{sec:multiscaleproblems}
With the definitions of the multiscale family of coarser spaces ${\Xsp}_j$ and
${\Ysp}_j$, corresponding measures $\mu_j$ and $\nu_j$, and corresponding cost
$\cost_j$, we may consider, for each scale $j$, the following optimal transport
problem:
\begin{equation}
\begin{aligned}
\coupling_j^*&: = \argmin_\coupling\!\!\!\sum_{\substack{k=1,\dots,K_j^{{\Xsp}}\\ k'=1,\dots,K_j^{{\Ysp}}}} 
              \!\!\!\!\cost_j(x_{j,k}, y_{j,k'}) \coupling(x_{j,k}, y_{j,k'}) 
              \,\,\text{s.t.}\,\,
\begin{cases}
\sum_{k'} \coupling(x_{j,k}, y_{j,k'}) = \mu_j(\{x_{j,k}\}) & \forall k\in K_j^\Xsp \\ 
\sum_{k} \coupling(x_{j,k}, y_{j,k'}) = \nu_j(\{y_{j,k'}\}) & \forall k'\in K_j^\Ysp\\
 \coupling(x_{j,k},y_{j,k'})\ge 0
\end{cases}
\label{e:LPformulation_j}
\end{aligned}
\end{equation}

The problems in this family are related to each other, and to the optimal transportation
problem in the original spaces. We define the cost of a coupling as
\begin{equation}
\couplingcost(\coupling_j) = \sum_{\substack{k=1,\dots,K_j^{{\Xsp}}\\
k'=1,\dots,K_j^{{\Ysp}}}}  \cost_j(x_{j,k}, y_{j,k'}) \coupling_j(x_{j,k},
y_{j,k'}) \, . 
\end{equation}
The cost of the optimal coupling $\coupling_j^*$ at scale $j$ is provably an
approximation to the cost of the optimal coupling $\coupling^*$ (which is equal
to $\coupling_J^*$):

\begin{proposition}
Let $\coupling^*$ be the optimal coupling, i.e. the solution to
\eqref{e:LPformulation}, and $\coupling^*_j$ the optimal coupling at scale $j$,
i.e. the solution to \eqref{e:LPformulation_j}. Define 
\begin{equation}
E_j(\coupling^*) := \sum_{\substack{k=1,\dots,K_j^{{\Xsp}}\\ k'=1,\dots,K_j^{{\Ysp}}}} 
  \sum_{\substack{x\in C^{{\Xsp}}_{j,k}\\ y\in C^{{\Ysp}}_{j,k'}}}
  \left(\cost_j(x_{j,k},y_{j,k'})-\cost(x,y)\right) \coupling^*(x,y)\,.
\end{equation}
Then
\begin{equation}
\couplingcost(\coupling_j^*)\le\couplingcost(\coupling^*)+E_j(\coupling^*)\,,
\label{e:UBcouplingcostpij}
\end{equation}
and if $\cost_j=\cost$ and $\cost$ is Lipschitz with constant $||\nabla\cost||_\infty$, we have 
\begin{equation}
\couplingcost(\coupling_j^*)\le\couplingcost(\coupling^*)+2^{-j}A||\nabla \cost||_\infty\,,
\label{e:UBcouplingcostpijboundednabld}
\end{equation}
where $A$ is such that
$\max_{k,k'}\{\mathrm{diam}(C^{{\Xsp}}_{j,k}),\mathrm{diam}(C^{{\Ysp}}_{j,k'})\}\le
A\cdot2^{-j}$.
\end{proposition}

\begin{proof}
Consider the coupling $\coupling_j$ induced at scale $j$ by the optimal
coupling $\coupling^*$, defined by 
\begin{equation*}
\coupling_j(x_{j,k},y_{j,k'}) = \sum_{\substack{x\in C^{{\Xsp}}_{j,k}\,,
 y\in C^{{\Ysp}}_{j,k'}}} \coupling^*(x,y)\,.
\end{equation*}
First of all, since $\{C^{{\Xsp}}_{j,k}\}_k$ and $\{C^{{\Ysp}}_{j,k'}\}_{k'}$
are partitions, it is immediately verified that $\coupling_j$ is a coupling. Secondly,
observe that 
\begin{align*}
\couplingcost(\coupling_j)
&=\sum_{\substack{k=1,\dots,K_j^{{\Xsp}}\\ k'=1,\dots,K_j^{{\Ysp}}}} 
  \cost_j(x_{j,k}, y_{j,k'}) \coupling_j(x_{j,k}, y_{j,k'})
=\sum_{\substack{k=1,\dots,K_j^{{\Xsp}}\\ k'=1,\dots,K_j^{{\Ysp}}}} 
  \sum_{\substack{x\in C^{{\Xsp}}_{j,k}\\ y\in C^{{\Ysp}}_{j,k'}}} 
  \cost_j(x_{j,k}, y_{j,k'}) \coupling^*(x,y)\\
&=\sum_{\substack{x\in{\Xsp}\\ y\in{\Ysp}}} \cost(x,y)\coupling^*(x,y)+
  \sum_{\substack{k=1,\dots,K_j^{{\Xsp}}\\ k'=1,\dots,K_j^{{\Ysp}}}} 
  \sum_{\substack{x\in C^{{\Xsp}}_{j,k}\\ y\in C^{{\Ysp}}_{j,k'}}}
  \left(\cost_j(x_{j,k},y_{j,k'})-\cost(x,y)\right) \coupling^*(x,y)\\
&= \couplingcost(\coupling^*) + \underbrace{ \sum_{\substack{k=1,\dots,K_j^{{\Xsp}}\\
   k'=1,\dots,K_j^{{\Ysp}}}} \sum_{\substack{x\in C^{{\Xsp}}_{j,k}\\ y\in C^{{\Ysp}}_{j,k'}}}
  \left(\cost_j(x_{j,k},y_{j,k'})-\cost(x,y)\right) \coupling^*(x,y)}_{=:E_j(\coupling^*)}\,.
\end{align*}
Since $\couplingcost(\coupling^*_j)\le\couplingcost(\coupling_j)$ (since
$\coupling^*_j$ is optimal), we obtain \eqref{e:UBcouplingcostpij}.
When $\cost_j=\cost$ and $\cost$ is Lipschitz with constant $||\nabla\cost||_\infty$, we have
\begin{align*}
E_j(\coupling^*)
&\le \sum_{\substack{k\,:\,x\in C^{{\Xsp}}_{j,k}\\ 
  k'\,:\,y\in C^{{\Ysp}}_{j,k'}}} \sum_{\substack{k=1,\dots,K_j^{{\Xsp}}\\ 
  k'=1,\dots,K_j^{{\Ysp}}}}||\nabla\cost||_\infty\cdot ||(x_{j,k},y_{j,k'})-(x,y)||\coupling^*(x,y)\\
&\le \sum_{\substack{k\,:\,x\in C^{{\Xsp}}_{j,k}\\ 
  k'\,:\,y\in C^{{\Ysp}}_{j,k'}}} \sum_{\substack{k=1,\dots,K_j^{{\Xsp}}\\ 
  k'=1,\dots,K_j^{{\Ysp}}}}||\nabla\cost||_\infty 2^{-j}A\coupling^*(x,y)\\
&\le 2^{-j}A||\nabla\cost||_\infty \,,
\end{align*}
with $A$ as in the claim.
\end{proof}


In the discrete, finite case that we are considering, $\mu_j\rightarrow\mu$ and
$\nu_j\rightarrow\nu$ trivially since $\mu_J=\mu$ and $\nu_J=\nu$ by
construction. If $\mu$ and $\nu$ were continuous, and at least when
$\cost(x,y)=\rho(x,y)^p$ for some $p\ge1$, then if $\mu$ and $\nu$ have finite
$p$-moment (i.e. $\int_\Xsp \rho(x,x_0)^pd\mu$ and similarly for $\nu$), we
would obtain convergence of a subsequence of $\cost(\coupling^*_j)$ to $\cost(\coupling^*)$ by
general results (e.g. as a simple consequence of Lemma 4.4. in ~\citep{villani:book2009}).

Note that the approximations do not guarantee that the transport plans are
close in any other sense but their cost. Consider the arrangement in
Figure~\ref{fig:counter-approximation}, the transport plans are
$\epsilon$-close in cost but the distances between the target locations of the
sources are far no matter how small $\epsilon$ gets.  
\begin{figure}[thb]
\centering
\begin{tabular}{VV}
\includegraphics[width=0.99\linewidth]{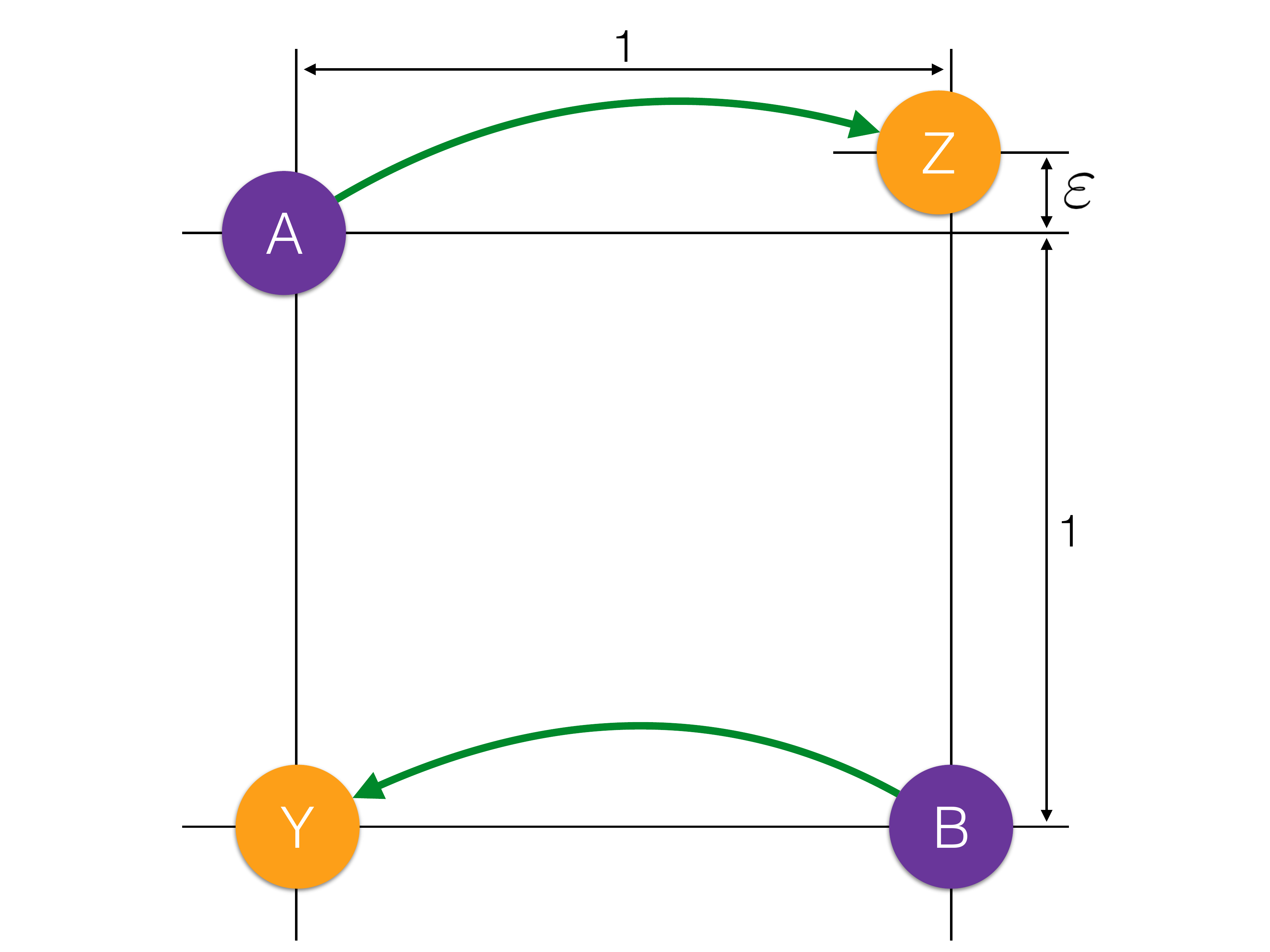} &
\includegraphics[width=0.99\linewidth]{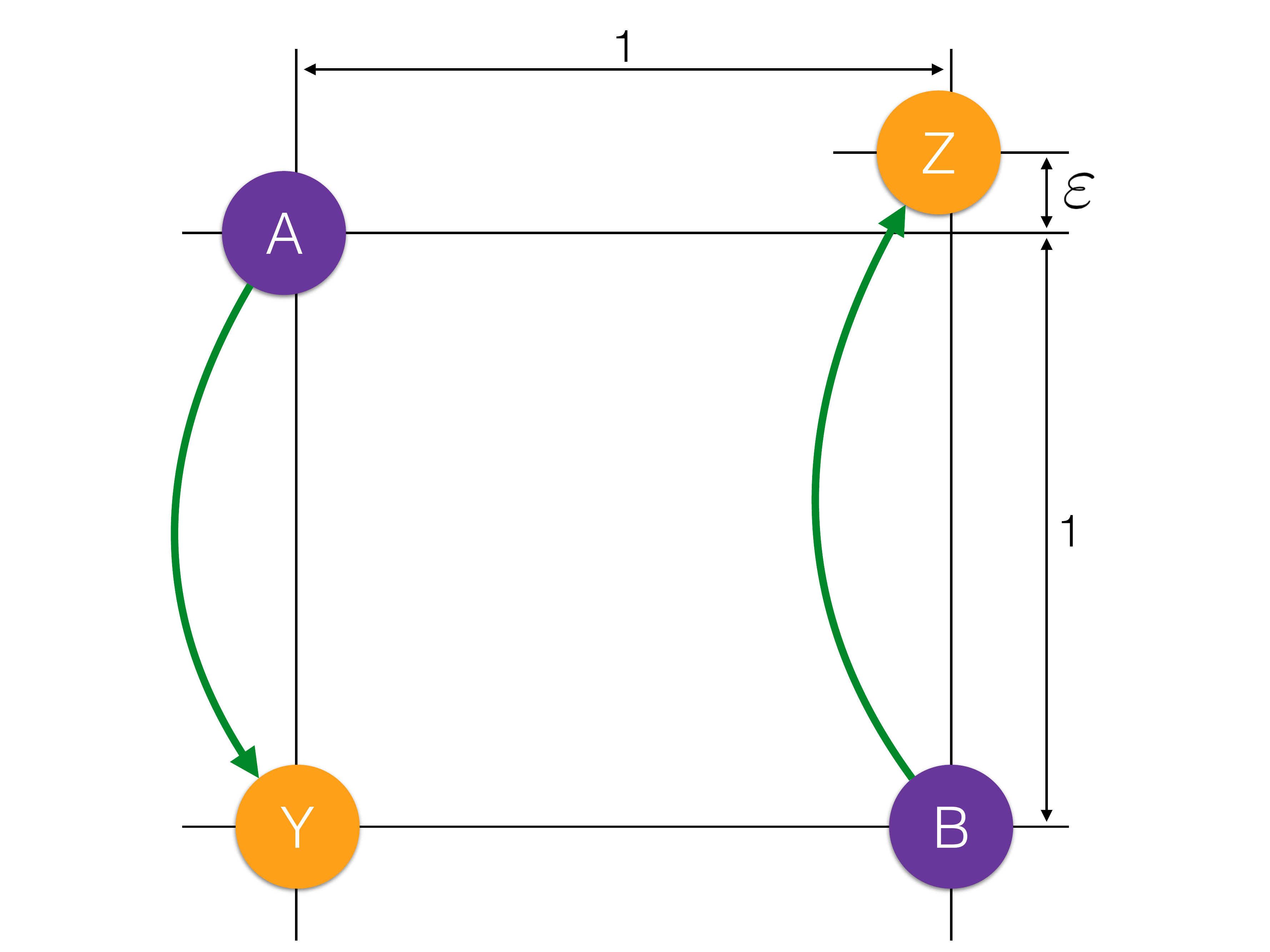} \\
                      (a) & (b)  
\end{tabular}
  \vspace{-0.1in}
\caption{
\label{fig:counter-approximation}
An example that illustrates that closeness in cost does not indicate closeness
of the transport plan. The transport plans (green arrows) between the sources A
and B (purple) and the targets Y and Z (orange)  in (a) and (b) are
$\epsilon$-close but their respective target locations are very far.  } 
\end{figure}

\subsection{Propagation Strategies}
\label{sec:propagation}
The approximation bounds in Section~\ref{sec:multiscaleproblems} show that the
optimal solution $\coupling_{j}$ at scale $j$ is $| E_{j+1} - E_j |$ close to
optimal solution $\coupling_{j+1}^*$ at scale $j+1$. This suggests that the
solution at scale $j$ can provide a reasonably close to optimal initialization
for scale $j+1$.  

As proposed by~\citet{glimm:arxiv2011} the solution at a given scale can be
interpolated at the next scale (or finer discretization).  The most direct
approach to initialize the transport problem at scale $j+1$ given the solution
$\coupling_{j}^*$ at scale $j$ is to distribute the mass $\coupling_{j}^*(x_k,
y_{k'})$ equally to all combinations of paths between $\children(x_k)$ and
$\children(y_{k'})$.  

This propagation strategy results in a reduction in the number of iterations
required to find an optimal solution at the subsequent scale.  This
warm-starting alone is often not sufficient, however. At the finest scale the
problems still requires the solution of a problem of size $O(n^2)$. This
quickly reaches memory constraints with $\Omega(10^4)$ points, and a single
iteration of Newton's method or a pivot step of a linear program becomes
prohibitively slow. Thus, we consider reducing the number of variables, which
substantially speeds up the algorithm, albeit we may lose guarantees on its
computational complexity and/or its ability to achieve arbitrary accuracy, so
that only numerical experiments will support our constructions.  These
reductions are achieved by considering only a subset $\mathbf{R}_{j+1}$ of all
possible paths $\mathbf{A}_{j+1}$ at scale $j+1$.

To distinguish the optimal solution on the reduced set of paths at scale $j$
from the optimal solution over all paths we introduce some notation. Let
$\mathbf{A}_{j+1}$ be the set of all possible paths between sources and targets
at scale $j+1$. Let $\mathbf{R}_{j+1} \subseteq \mathbf{A}_{j+1}$ be the set of
paths propagated from the previous solution (e.g. children of mass-bearing
paths found at scale $j$).  Let $\coupling^*_{j+1}  |_{\mathbf{P}}$ be the
optimal solution to the transport problem restricted to a set of paths
$\mathbf{P} \subset \mathbf{A}_{j+1}$.  With this notation, $ \coupling_{j+1}^*
|_{\mathbf{A}_{j+1}} = \coupling_{j+1}^*$.  The optimal coupling
$\coupling_{j+1}^*|_{\mathbf{R}_{j+1}}$ on the reduced set of paths problem does
not need to match the optimal coupling $\coupling_{j+1}^*$ on all paths.
However $ \coupling_{j+1}^* |_\mathbf{R_{j+1}}$ does provide a starting point
for further refinements discussed in Section~\ref{sec:refinement}.

\subsubsection{Simple Propagation}
\label{sec:simple}
The most direct approach to reduce the number of paths considered at subsequent
scales is to include only paths at scale $j+1$ whose endpoints are children of
endpoints of mass-bearing paths at scale $j$. The optimal solution at scale $j$
has exactly $K_j^{\Xsp} + K_j^{\Ysp} - 1$ paths with non-zero weight. Thus, the
number of paths at scale $j+1$ reduces to $O( C^2(K_{j}^{\Xsp} + K_j^{\Ysp}))$,
where $C$ is the maximal number of children of any node at scale $j$. In
particular, $C\asymp 2^{d}$ for a doubling space of dimension $d$. This reduces
the number of variables from ``quadratic'', $O(K_{j}^{\Xsp} K_j^{\Ysp} )$, to
linear,  $O(K_{j}^{\Xsp}+K_j^{\Ysp} )$.  This propagation strategy by itself,
however, often leads to a dramatic loss of accuracy in both the cost the
transportation plan, and the transportation plan itself.

\subsubsection{Capacity Constraint Propagation}
\label{sec:capacity}
This propagation strategy solves a modified minimum flow problem at scale $j$
in order to include additional paths at scale $j+1$ that are likely to be
included in the optimal solution $\coupling_{j+1}^*$. This is achieved by
adding a capacity constraint to the mass bearing paths at scale $j$ in the
optimal coupling  $ \coupling_j^* |_{\mathbf{R}_j}$: The amount of mass of a
mass bearing path $ \coupling_j^* |_{\mathbf{R}_j}( x_{j,k}, y_{j,k'} )$ is
constrained to $\lambda \min \left ( \mu_j(x_{j,k}), \nu_j(y_{j,k'})  \right)$
with $\lambda$ random uniform on $[0.1, 0.9]$. The randomness is introduced to
avoid degenerate constraints.  The solution of this modified minimum--flow
problem forces the inclusion of $n_c$ additional paths, where $n_c$ is the
number of constraints added. There are various options for adding capacity
constraints, we propose to constrain all mass bearing paths of the optimal
solution at scale $\coupling_j^*|_{\mathbf{R}_j}$. The capacity constrained
problem thus results in a solution with twice the number of paths as in the
coupling $\coupling_j^*|_{\mathbf{R}_j}$. The solution of the capacity
constrained minimum flow problem is propagated as before to the next scale.

To increase the likelihood of including paths required to find an optimal
solution at the next scale, the capacity constrained procedure can be iterated
multiple times.  Each time the mass bearing paths in the modified solution are
constrained and a new solution is computed. Each iteration doubles the number
of mass bearing paths and the number of iterations controls how many paths are
propagated to the next scale. Thus, the capacity constraint propagation
strategy bounds the number of paths considered in the linear program. The
optimal transport plan on a source set $\Xsp$ and $\Ysp$ results in a linear
program with $|\Xsp| + |\Ysp|$ constraints and $|\Xsp| |\Ysp|$ variables and
the optimal transport plan has $|\Xsp| + |\Ysp| - 1$  mass bearing paths. It
follows that the capacity constraint propagation strategy considers linear
programs with at most $O\left( 2^i (|\Xsp| + |\Ysp|) \right)$ constraints,
where $i$ is the number of iterations of the capacity propagation scheme. This
results in a significant reduction in problem size, since at each scale we only
consider a number of paths scaling as $O(|\Xsp|+|\Ysp|)$ instead of $O(|\Xsp|
|\Ysp|)$.

\subsection{Refinement Strategies}
\label{sec:refinement}
Solving the reduced transport problem at scale $j+1$, propagated from scale
$j$, can not guarantee an optimal solution at scale $j+1$. Propagating a
sub-optimal solution further may lead to an accumulation of errors. This
section describes strategies to refine the reduced transport problem to find
closer approximations or even optimal transport plans at each scale.  These
refinement strategies are essentially batch column generation
methods~\citep{desaulniers:book2002}, that take advantage of the multiscale
structure.

\subsubsection{Potential Refinement}
\label{sec:potential}
This refinement strategy exploits the potential functions, or dual solution, to
determine additional paths to include given the currently optimal solution on
the reduced set of paths from the propagation. The dual formulation of the
optimal transport can be written as:
\begin{equation}
  \max_{\phi, \psi} 
  \sum_{i=1,\dots,n} \mu(\{x_i\}) \phi(x_i) - 
  \sum_{j=1,\dots,m} \nu(\{y_j\}) \psi(y_j)
  \quad \text{s.t.} \quad 
  \phi(x_i) - \psi(y_j) \leq \cost(x_i, y_j) \,.
\label{e:DualLPformulation}
\end{equation}
The functions $\phi$ and $\psi$ are called dual variables or potential
functions.  From the dual formulation it follows that at an optimal solution
the {\em reduced cost} $\cost(x, y) - \phi(x) + \psi(y)$ is larger or equal to
zero. This also follows from the Kantorovich duality of optimal
transport~\citep[Chapter~5]{villani:book2009}.

The potential refinement strategy uses the potential functions $\phi$ and
$\psi$ from the solution of the reduced problem to determine which additional
paths to include. If the solution on the reduced problem is not optimal on all
paths, then there exist paths with negative reduced cost. Thus, we check the
reduced cost between all paths and include the ones negative reduced cost.  A
direct implementation of this strategy would require to check all possible
paths between the source and target points. To avoid checking all pairwise
paths between source and target point sets at the current scale, we introduce a
branch and bound procedure on the multiscale structure that efficiently
determines all paths with negative reduced cost. 

Let $ \phi_j^* |_\mathbf{P}$ and $ \psi_j^*|_\mathbf{P}$ the dual variables,
i.e., the potential functions, at the optimal solution
$\coupling_j^*|_\mathbf{P}$ on the set of paths $\mathbf{P}$ for the source and
target points, respectively. Define $\mathbf{V}_j( \coupling_j^*|_\mathbf{P} )$
as the set of paths with non-positive reduced cost with respect to $ \phi_j^*
|_\mathbf{P}$ and $\psi_j^* |_\mathbf{P}$
$$
\mathbf{V}_j( \coupling_j^* |_\mathbf{P}) = \left \{ \coupling_j(x_k, y_{k'})
\in \mathbf{A}_j : \cost(x_k, y_{k'}) -  \phi_j^* |_\mathbf{P}(x_k) - \psi_j^*
|_{\mathbf{P}}(y_{k'}) \le 0 \right \} \, .
$$
The {\em potential} refinement strategy now refines the propagated solution
$\coupling_j^*|_{\mathbf{R}_j}$ by including all negative reduced cost paths
$\mathbf{Q}_j^0 = \mathbf{V}_j( \coupling_j^* |_{\mathbf{R}_j})$.
Let $\coupling_j^*|_{\mathbf{Q}_j^0}$ be the associated optimal transport. This
new solution changes the potential functions which in turn may require to
include additional paths. Thus the potential refinement strategy can be
iterated with $\mathbf{Q}_j^i = \mathbf{V}_j(
\coupling_j^*|_{\mathbf{Q}_j^{i-1}})$ leading to monotonically decreasing
optimal transport plans $\coupling_j^*|_{Q_j^i}$. Since a solution is optimal
if and only if all reduced cost are $\ge 0$ this iterative strategy converges
to the optimal solution.  

The set of paths with negative reduced cost given $ \phi_j^*|_\mathbf{P}$ and $
\psi_j^* |_\mathbf{P}$ are determined by descending the tree and excluding
nodes that can not contain any negative reduced cost paths. This requires
bounds on the potential functions for any node at scale smaller than $j$. The
bound is achieved by storing at each target node of the multiscale
decomposition the maximal $  \psi_j^* |_\mathbf{P}$ of any of its descendants
at scale $j$. Now for each source node $x_k$ the target multiscale
decomposition is descended, but only towards nodes at which a potential
negative reduced cost can exist.

Depending on the properties of $ \phi_j^* |_\mathbf{P}$ and $\psi_j^*
|_\mathbf{P}$, this refinement strategy may reduce the number of required cost
function evaluations drastically. In empirical experiments the total number of
paths considered in the iterative potential refinement was typically reduced to
$O( 2^d ( K_j^\Xsp + K_i^\Xsp) )$, with $d$ being the doubling dimension of the
data set. 

The potential refinement strategy is also employed by~\citet{glimm:arxiv2011}
without the branch and bound procedure. The shielding neighborhoods proposed
by~\citet{schmitzer2015sparse} similarly uses the potentials to define which
variables are needed to define a sparse optimal transport problem and also
suggests to iterate the neighborhood shields in order to arrive at an optimal
solution.

\subsubsection{Neighborhood Refinement}
\label{sec:neighborhood}
This section presents  a refinement strategy that takes advantage of the
geometry of the data. The approach is based on the heuristic that most paths at
the next scale are sub-optimal due to boundary effects when moving from one
scale to the next, induced by the sharp partitioning of the space at scale $j$.
Such artifacts from the multiscale structure are mitigated by including paths
between spatial neighbors in the source and target locations of the optimal
solution on the propagated paths. This refinement strategy is also employed
by~\citet{oberman2015efficient}.

Let $\mathbf{N}_j\left(\coupling_j, r\right)$ be the set of paths such that the
source and target of any path are within radius $r$ of the source and target of
a path with non-zero mass transfer in $\coupling_j$.  The neighborhood
refinement strategy is to expand the reduced set of paths using the union of
paths in the current reduced set and its neighbors:
$$
\mathbf{E}_j = \mathbf{R}_j \cup \mathbf{N}_j \left( \coupling_j^*|_{\mathbf{R}_j}, r\right) \,
$$  
When moving from one scale to the next the cost of any path can change at most
two times the radius $r$ of the parent node which suggests to set the radius of
the neighborhood in consideration as two times the node radius. 

This heuristics does not guarantee an optimal solution, but does reduce the
number of paths to consider at scale $j$ to $O( q_r^2 (K_j^{{\Xsp}} +
K_j^{{\Ysp}}))$ with $q_r$ being the number of neighbors within radius
$2^{-j}r$, for a doubling space with dimension $d$, $q_r\asymp r^d$.

The neighborhood strategy requires to efficiently compute the set of points
within a ball of a given radius and location. Depending on the multiscale
structure there are different ways to compute the set of neighbors. A generic
approach that does not depend on any specific multiscale structure is to use a
branch and bound strategy. For this approach, each node requires an upper bound
$u(c_i)$ on the maximal distance from the representative of node $c_i$ to any of
its descendants. Using these upper bounds the tree is searched, starting from
the root node, while excluding any child nodes for which $\cost(x, c_i) - u(c_i) >
r$ from further consideration.  For multiscale structures such as regular grids
more efficient direct computations are possible. For this paper we implemented
the generic branch and bound strategy that works with any multiscale structure.

\subsection{Remark on Errors}
The error induced by this multiscale framework stems from two sources.  First, if
$J_0<J$, i.e., the optimal transport problem is only solved up to scale $J_0$,
the solution at scale $J_0$ has a bounded approximation error, as detailed in
Section~\ref{sec:multiscaleproblems}.  By solving the transport problem only up
to a certain scale permits to trade-off computation time versus precision with
guaranteed approximation bounds.  However, to speed up computation we rely on
heuristics that, depending on the refinement strategy, yield solutions at each
scale that might not be optimal. This second type error is difficult to
quantify; however, for the potential refinement strategy that we introduce
Section~\ref{sec:potential}, an optimal solution can be guaranteed.  The
propagation and refinement strategies introduced in
Sections~\ref{sec:propagation} and~\ref{sec:refinement} permit trade-offs
between accuracy and computational cost.

\section{Numerical Results and Comparisons}
\label{sec:results}
We tested the performance of the proposed multiscale transport framework with
respect to computational speed as well as accuracy for different propagation
and refinement strategies on various source and target point sets with
different properties.

The multiscale approach is implemented in C++ with an \textsf{R} front end in
the package {\em mop}\footnote{available on
\url{https://bitbucket.org/suppechasper/optimaltransport}}. For comparisons, we also
implemented the Sinkhorn distance approach~\citep{cuturi:nips2013} in C++ with
an \textsf{R} front end in the package {\em
sinkhorn}\footnotemark[\value{footnote}].  Our C++ implementation uses the
Eigen linear algebra library~\citep{eigen}, which resulted in faster runtimes
than the MATLAB implementation by~\citet{cuturi:nips2013}.

\subsection{Implementation Details}
The \textsf{R} package  {\em mop} provides a flexible implementation of the
multiscale framework and is adaptable to any multiscale decomposition that can
be represented by a tree. The package permits to use different optimization
algorithms to solve the individual transport problem.  Currently the multiscale
framework implementation supports the open source library GLPK~\citep{glpk} and
the commercial packages  MOSEK~\citep{mosek} and CPLEX~\citep{cplex} (both free
for academic use). MOSEK and CPLEX support a specialized network simplex
implementation that runs 10-100 times faster than a typical primal simplex
implementation. Both the MOSEK and CPLEX network simplex run at comparable
times, with CPLEX slightly faster in our experiments.  Furthermore CPLEX
supports starting from an advanced initial basis for the network simplex
algorithm which improves the multiscale run-times significantly.  Thus, the
numerical test are all run using the CPLEX network simplex algorithm.

\subsubsection{Algorithms for Constructing Multiscale Point Set Representations}
\label{s:multiscalepointsets}
Various approaches exist to build the multiscale structures described
in Section~\ref{sec:coarsening}, such as hierarchical clustering
type algorithms~\citep{ward:jasa1963}, or in low dimensions constructions such
as quad and oct-trees~\citep{finkel:ai197,jackins:cgip1980} or
kd-trees~\citep{bentley:acm1975} are feasible. Data structures developed for
fast nearest neighbor queries, such as navigating
nets~\citep{krauthgamer:soda2004} and cover trees~\citep{beygelzimer:icml2006}
induce a hierarchical structure on the data sets with guarantees on partition
size and geometric regularity of the elements of the partition at each scale,
under rather general assumptions on the distribution of the points.  The
complexity of cover trees~\citep{beygelzimer:icml2006} is $O(C^d D n\log n)$,
for some constant $C$, where $n=|{\Xsp}|$, $d$ is the doubling dimension of
${\Xsp}$, and $D$ is the cost of computing a distance between a pair of points.
Therefore the algorithm is practical only when the intrinsic dimension is
small, in which case they are provably adaptive to such intrinsic dimension.
The optimal transport approach does not rest on a specific multiscale structure
and can be adapted to application-dependent considerations. However, the
properties of the multiscale structure, i.e., depth and partition sizes, do
affect run-time and approximation bounds.

In our experiments we use an iterative $K$-means strategy to recursively split
the data into subsets. The tree is initialized using the mean of the complete
data set as the root node.  Then K-means is run resulting in $K$ children. This
procedure is recursively applied for each leaf node, in a breadth first
fashion, until a desired number of of leaf nodes, a maximal leaf radius or the
leaf node contains only a single point. For the examples shown we select $K =
2^d$ with $d$ the dimensionality of the data set. For high-dimensional data $K$
could be set to an estimate of the intrinsic dimensionality. Since in all
experiments we are equipped with a metric structure we use the pointwise
coarsening of the cost function as in Equation~\eqref{e:costpointwise}. The
reported results include the computation time for building the hierarchical
decomposition, which is, however, negligible compared to solving the transport
problem at all scales.

\subsubsection{Multiscale Transport Implementation for Point Sets}  
Algorithm~\ref{f:algo-detailed} details the steps for computing multiscale
optimal transport plans using the newtork simplex for solving the optimization
problems at each scale and iterated $K$--means to construct the multiscale
structures.  
\begin{algorithm}[thb]
  \KwIn{ Source point set $\Xsp = \{x\}_{i=0}^N$.

         Target point set $\Ysp = \{y\}_{i=0}^M$. 

         A propagation strategy $p$. 

         A refinement strategy $r$.
       } 
 \KwOut{Multiscale family of transport plans
   $(\coupling_j: \Xsp_j\rightarrowtail \Ysp_j)_{j=0}^{J}$}

Construct multiscale point sets $\{ \Xsp_j \}_{j=0}^N$ and
$\{ \Ysp_j \}_{j=0}^M$ using iterated K--means. 

\If{ $N < M$ }{
  Add scales $\{ \Xsp_j \}_{j=N+1}^J$ by repeating the last scale.
}  
\If{ $N > M$ }{
  Add scales $\{ \Ysp_j \}_{j=M+1}^J$ by repeating the last scale.
}  

Set $J = \max(N, M)$
  
Form the measures $\mu_j$ and $\nu_j$ as in equation~\eqref{e:muj}.

Compute optimal transport $\coupling_0$ at coarsest scale with the network simplex.

 \For{ $j=1\dots J$ }{
     Propagate the $\coupling_{j-1}$ from scale $j-1$ to scale $j$ using the
     propagation strategy $p$, obtaining a set of paths $\mathbf{S}_j =
     p(\coupling_{j-1}$ at scale $j$ 
     \vspace{0.1in}

     Use the network simplex algorithm to solve for optimal transport on the
     set of paths $\mathbf{S}_j$ yielding a coupling $\tilde \coupling_j$ 
   
     Create the refined set of paths $\mathbf{R}_j = r( \tilde \coupling_j )$
     using the refinement strategy $r$.  
   
     Use the network simplex algorithm to solve for optimal transport on the
     set of paths $\mathbf{R}_j$ yielding the optimal coupling $\coupling_j$ on
     the paths $\mathbf{R}_j$.

 }
      
\caption{Point Set based Multiscale Optimal Transport Implementation
 \label{f:algo-detailed}
}
\end{algorithm}

\subsection{Propagation and Refinement Strategy Performance}
Figure~\ref{fig:strategies} illustrates the behaviour of the different
propagation and refinement strategies on two examples: The ellipses example in
Figure~\ref{fig:multiscale-strategy} and Caffarelli's smoothness counter example
described in~\citet[Chapter~12]{villani:book2009} and illustrated in
Figure~\ref{fig:datasets}. The ellipse example consists of two uniform samples
(source and target data set) of size $n$ from the unit circle with normal
distributed noise added with zero mean and standard deviation $0.1$. The source
data sample is then scaled in the x-Axis by $1.3$ and in the y-Axis by $0.9$,
while the target data set is scaled in the x-Axis
by 0.9 and in the y-Axis by $1.1$. 
Caffarelli's example consists of two uniform samples on $[-1,1]^2$ of size $n$.
Any points outside the unit circle are then discarded. Additionally, the target
data sample is split along the x-Axis at $0$ and shifted by $+2$ and $-2$ for
points with positive and negative x-Axis values, respectively. 
\begin{figure}[htb]
\centering
\begin{tabular}{VV}
\includegraphics[width=0.9\linewidth]{scale-9} & 
\includegraphics[width=0.9\linewidth]{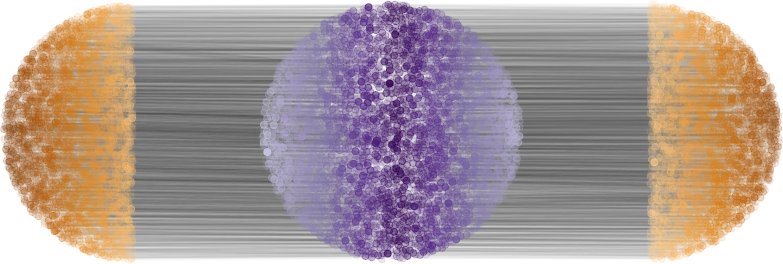}\\
                      (a) & (b)  
\end{tabular}
\vspace{-0.2in}
\caption{
\label{fig:datasets}
Optimal transport plans on the (a) ellipse and (b) Caffarelli data sets with
5000 points for the source and target point set each. The optimal transport
plans are indicated by black lines with transparency indicating the amount of
mass being transported from source to target. } 
\end{figure}
\vspace{-0.2in}
We tested 7 different strategies of different combinations of propagation
strategies from Section~\ref{sec:propagation} with refinement strategies from
Section~\ref{sec:refinement}:
\begin{compactenum}
  \item (ICP) Iterated capacity propagation (Section~\ref{sec:capacity}) with
    $0$ to $5$ iterations with no refinements. Note iterated capacity with $0$
    iterations is equivalent to simple propagation (Section~\ref{sec:simple}).
  \item (NR) Simple propagation (Section~\ref{sec:simple}) with neighborhood
    refinement (Section~\ref{sec:neighborhood}) with radius factor ranging from
    $0.5$ to $2.5$ in $0.5$ increments.
  \item (ICP + NR) Iterated capacity propagation with $1$ to $5$ iterations
    combined with neighborhood refinement with radius factor fixed to $1$.
  \item (CP + NR) A single iteration of capacity propagation combined with
    neighborhood refinement with radius factor from $0.5$ to $2.5$ with $0.5$
    increments.
  \item (IPR) Simple propagation with iterated potential refinement
    (Section~\ref{sec:potential}) with $1$ to $5$ iterations.  
  \item (ICP + PR) Iterated capacity propagation with $1$ to $5$ iterations
    combined with a single potential refinement step.
  \item (CP + IPR) A single iteration of capacity propagation combined with
    $1$ to $5$ iterations of potential refinement.
\end{compactenum}

Figure~\ref{fig:strategies} shows that almost all strategies have less than one
percent relative error. The exception is the simple propagation with no
refinements applied (i.e. iterated capacity propagation with no iterations).
The randomness of the iterated capacity constrained algorithm can results in worse
results despite a larger number of iterations. However, it will always perform better than
using the simple propagation strategy.  The neighborhood refinement strategy
improves the results significantly, but after a radius factor of one the improvements start to
level out.  The potential refinement strategy finds the optimal solution when
iterated a few times.  Combining the refinement strategy with capacity
propagation reduces relative error and computation time.  The computation time
is reduced because the capacity propagation yields a closer initialization to
the optimal solution. The combination of potential refinement strategy and capacity propagation has
the additional benefit that the branch and bound strategy is more efficient
since fewer comparisons need to be made when checking for paths and smaller
linear programs have to be solved in each refinement step.
\begin{figure}[htb]
\centering
\begin{tabular}{VV}
\includegraphics[width=0.85\linewidth]{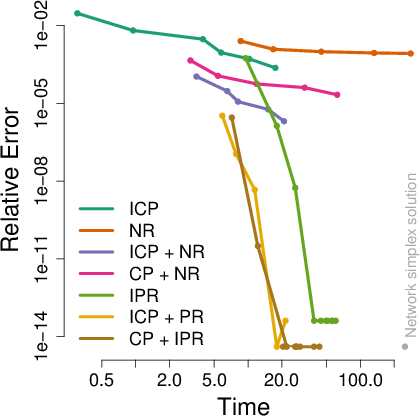} &
\includegraphics[width=0.85\linewidth]{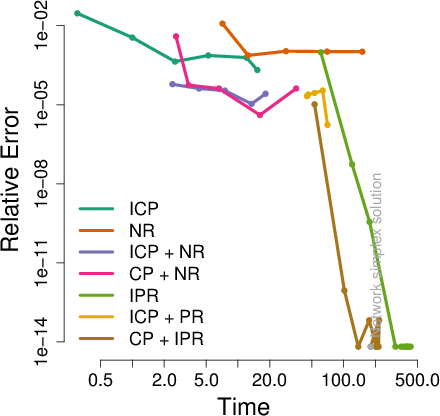} \\
                      (a) & (b)  
\end{tabular}
  \vspace{-0.15in}
\caption{
\label{fig:strategies}
A comparison of different multiscale propagation and refinement strategies on
the (a) ellipse and (b) Caffarelli data sets with 5000 points for the source and
target point set each. For a description of the different strategies see text,
the lines in the relative error graphs indicate increasing number of iterations
or radius factor for the different strategies.  Both the time and accuracy axes
are in logarithmic scale. All strategies, except capacity propagation with $0$
iterations (ICP) and neighborhood refinement (NR) with radius factor 0.5, find
solution with relative error less than one percent. The capacity propagation
strategies can find optimal solutions up to numerical precision.} 
\end{figure}
  \vspace{-0.15in}
For very small relative error $( < 10^{-11} )$ the optimal solution is
sometimes not achieved. This is due to tolerance settings in the network
simplex algorithm which are on the order of $10^{-11}$. Thus, depending how the
network simplex approaches the optimum it might stop at different basic
feasible solutions.  

Figure~\ref{fig:strategies} shows that the computation time and relative
error for problem of equal sizes depends on the type of problem. An important
aspect of the problem type is the ratio of the transport distance to the diameters
of the source and target point sets. To illustrate this effect we computed
optimal transport plans for two data sets with source and target
distributions uniform on $[0,1]^2$. In the first case the distributions are
perfectly overlapping and in the second case the target distributions shifted
by 2 units in the x direction.  Figure~\ref{fig:distance-strategies} shows that
for large ratios the relative error is typically much smaller.  This is
expected since variations in the transport plan only change the cost
marginally. For small ratios, i.e.  source and target distributions that are
almost identical, a small variation in the transport plan leads to a large
relative error.
\begin{figure}[htb]
\centering
\begin{tabular}{cc}
                       \vspace{-0.3in} \\
\includegraphics[width=0.47\linewidth]{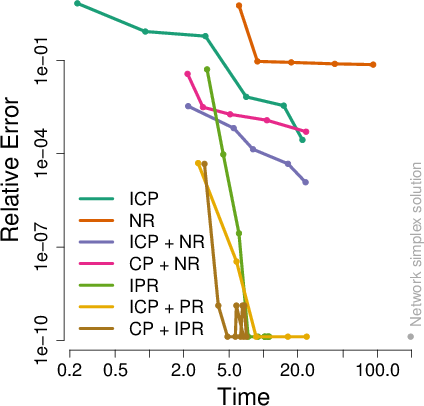} &
\includegraphics[width=0.47\linewidth]{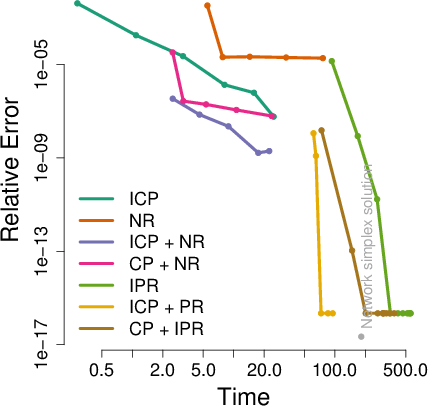} \\
                       \vspace{-0.2in} \\
                       (a) & (b)  
\end{tabular}
  \vspace{-0.15in}
\caption{
\label{fig:distance-strategies}
The effect of distance to support size ratio on computation time and relative
error for the different multiscale propagation and refinement strategies source
and target point sets of $5000$ points each sampled uniformly from a square of
side length one and ground truth transport distance (a) $0$ and (b) $2$.  Both
the time and accuracy axes are in logarithmic scale.  } 
\end{figure}
Another observation is the following: the potential strategy is less effective
for transport problems where most mass is transported very far, relative to the
distances within source and target point set. This is because a small change in
the length of a path can include many possible source and target locations and
the transport polytope has many suboptimal solutions with similar cost. If on
the other hand most mass is transported on the order of the nearest neighbor
distances within the source and target point set, there are many fewer possible
paths within a small change in path length and the transport polytope is
``steep'' in the direction of the cost function.

\subsection{Comparison to Network Simplex and Sinkhorn Transport}
In this section we compare the CPLEX network simplex algorithm~\citep{cplex}
and the Sinkhorn approach~\citep{cuturi:nips2013}  to three different
multiscale strategies: 
\begin{compactenum}
\item (CP)  Capacity propagation strategy using a single iteration with no
  further refinements.  
\item (CP + NR) Capacity propagation combined with neighborhood refinement
    with radius factor fixed to one and a single capacity constraint iteration.
\item (CP + PR) Capacity propagation combined with potential refinement with a
  single iteration each.
\end{compactenum}

\begin{figure}[htb]
\centering
\begin{tabular}{cc}
  \includegraphics[width=0.65\linewidth]{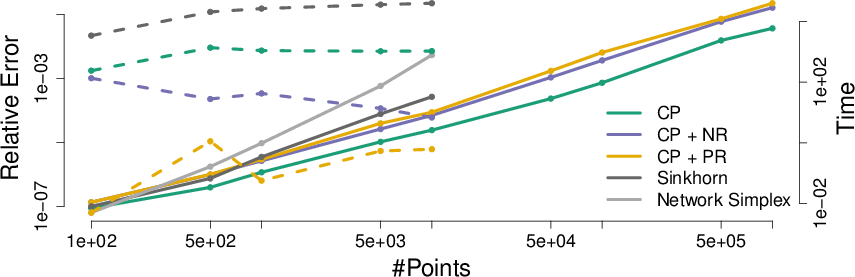} &
\includegraphics[width=0.28\linewidth]{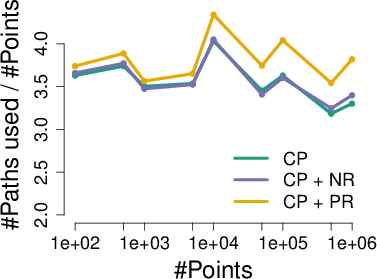} \\
(a) & (b)
\end{tabular}
\caption{
\label{fig:npoints}
(a) A comparison of computation time (solid) and relative error ( dashed ) with
respect to the network simplex solution for the CPLEX network simplex, Sinkhorn
distance and the proposed multiscale strategy with increasing number of points.
(b) The number of total paths considered is roughly constant with four times
the number of points in the problem, i.e., the multiscale approach results in
only a linear increase in problem size instead of quadratic for a direct
approach.  The number of points on the x-axis denotes the number of source
points $|\Xsp|$ which is approximately equal to the number of target points
$|\Ysp|$.  The Sinkhorn approach is competitive in computation time for smaller
problems.  For larger problems only the multiscale strategies outperform the
Sinkhorn approach quickly, and are the only algorithms that remain viable.
} 
\end{figure}
Figure~\ref{fig:npoints} shows computation time, relative error and problem
size for increasing sample size on the ellipse data set. The Sinkhorn transport
employs a regularization term and is thus not expected to converge to the
actual transport distance. The comparison of the relative error provides an
indication of how far the regularization strays from the true transport
distance. For small size problems the Sinkhorn transport is as comparable in
speed to the multiscale approach.. The multiscale approach outperforms both
competitors for moderate data sizes of a few thousand source and target points.
The multiscale approach scales roughly linear in the problem size while both
the Sinkhorn and network simplex scale approximately quadratically. The
capacity propagation without refinement runs an order of magnitude faster than
including refinements and results in relative errors less than one percent for
more than a few hundred points.  The network simplex and Sinkhorn approach run
out of memory on a 16GB computer for problems larger than around $2\cdot10^4$
constraints and about $10^8$ variables. The proposed multiscale framework
results in a linear increase in problem size for the propagation and refinement
strategies tested and, on the same 16GB computer, can solve instances with
$2\cdot10^6$ constraints (source and target points) in a few hours, which would
result in about $10^{12}$ variables (paths) for the full optimal transport
problem.  The computation times are comparable to the results reported
by~\citet{oberman2015efficient,schmitzer2015sparse} on examples on regular grids
that result in similar sized problems. 

The relative error and computation time depend again on the type of problem.
Figure~\ref{fig:dimensions}(a) shows computation time and relative error
on transport problems on source and target data sets sampled uniformly from
$[0,1]^2$ with different shifts in the target distribution. If required, the
relative error can be reduced using multiple capacity propagation and potential
refinement iterations as illustrated in Figure~\ref{fig:distance-strategies}.
In these experiments we set the tolerance parameter for Sinkhorn to $1e^{-5}$, and
tried also $1e^{-2}$ to see if this would result in significant speed ups, with
minimal accuracy loss, but it resulted in only approximately a $10\%$ speedup,
and no fundamental change in the behavior for large problem sizes. The
tolerance parameter for CPLEX is set to $1e^{-8}$. The approach by
~\citet{schmitzer2015sparse} is similar to the refinement property strategy, we
expect that computation time grows similarly as for the potential strategy for
problems with large transport distances compared to the source and target
support size.

The final experiment tests the performance with respect to the dimensionality of
the source and target distributions. We used two uniform distributions $[0,1]^d$
with $d$ varying from $2$ to $5$ and the target shifted such that the actual
transport distance is 0 and 2.
\begin{figure}[htb]
\centering
\begin{tabular}{ccc}
\includegraphics[width=0.32\linewidth]{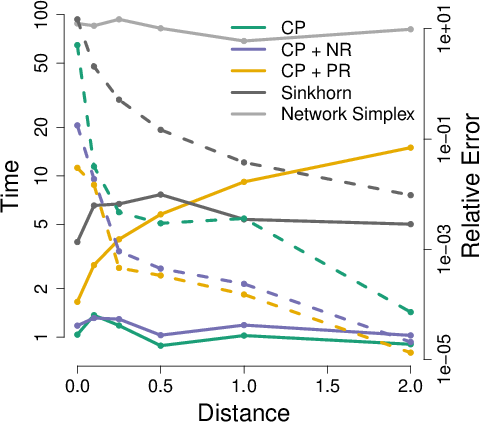} &
\includegraphics[width=0.32\linewidth]{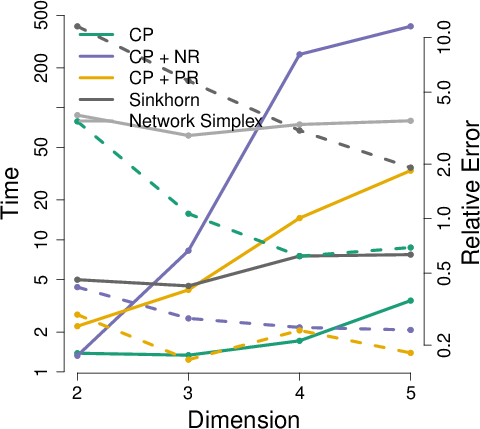} &
\includegraphics[width=0.32\linewidth]{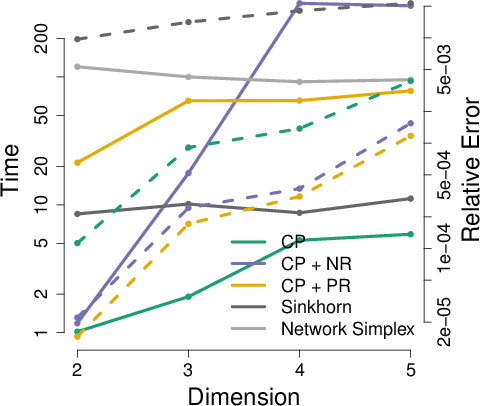} \\
                       \vspace{-0.2in} \\
                       (a) & (b)  & (c)
\end{tabular}
  \vspace{-0.1in}
\caption{
\label{fig:dimensions}
Computation time (solid) and relative error (dashed) with respect to (a)
changes transport distance to support size ratio in two dimensions and (b,c)
dimensionality with ground truth transport distance (b) 0 and (c) 2. The source
and target distributions are uniform on a square of side length 1 with
approximately 5000 points each. The potential refinement increases proportional
to the transport distance to support ratio. The neighborhood strategy is less
effective in higher dimensions, due to the curse of dimensionality, but
performs better than the Sinkhorn approach. The capacity propagation strategy
is less affected by the dimensionality of the problem. }  
\end{figure}

\section{Application to Brain MR Images}
\label{s:brainMR}
An important building block for the analysis of brain MRI populations is the
definition of a metric that measures how different two brain MRI are. A
mathematically well motivated and popular approach for distance computations
between brain images is based on large deformation diffeomorphic metric mappings
(LDDMM~\citep{miller2002metrics}). Here we explore optimal transport distance
as an alternative metric for comparing brain images.

To solve for the optimal transport map between two 3D brain images we extract
for each image a point cloud from the intensity volumes. Each point represents
a voxel as a point in 3-dimensional space, the location of the voxel. The mass
of the point is equal to the intensity value of the voxel, normalized to sum to
one over all points.  For illustration, Figure~\ref{fig:brains} shows a single
slice extracted from the original volumes and optimal transport maps between
the two slices. This 2D problem resulted in point set of approximately $20,000$
points. 
\begin{figure}[bth]
\centering
\begin{tabular}{cc}
                       \vspace{-0.1in} \\
\includegraphics[height=0.375\linewidth,angle=0]{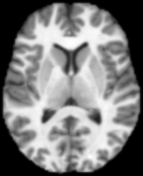} &
\includegraphics[height=0.375\linewidth,angle=0]{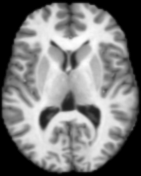} \\
                       \vspace{-0.1in} \\
                   (a) & (b) \\
\includegraphics[width=0.375\linewidth,angle=90]{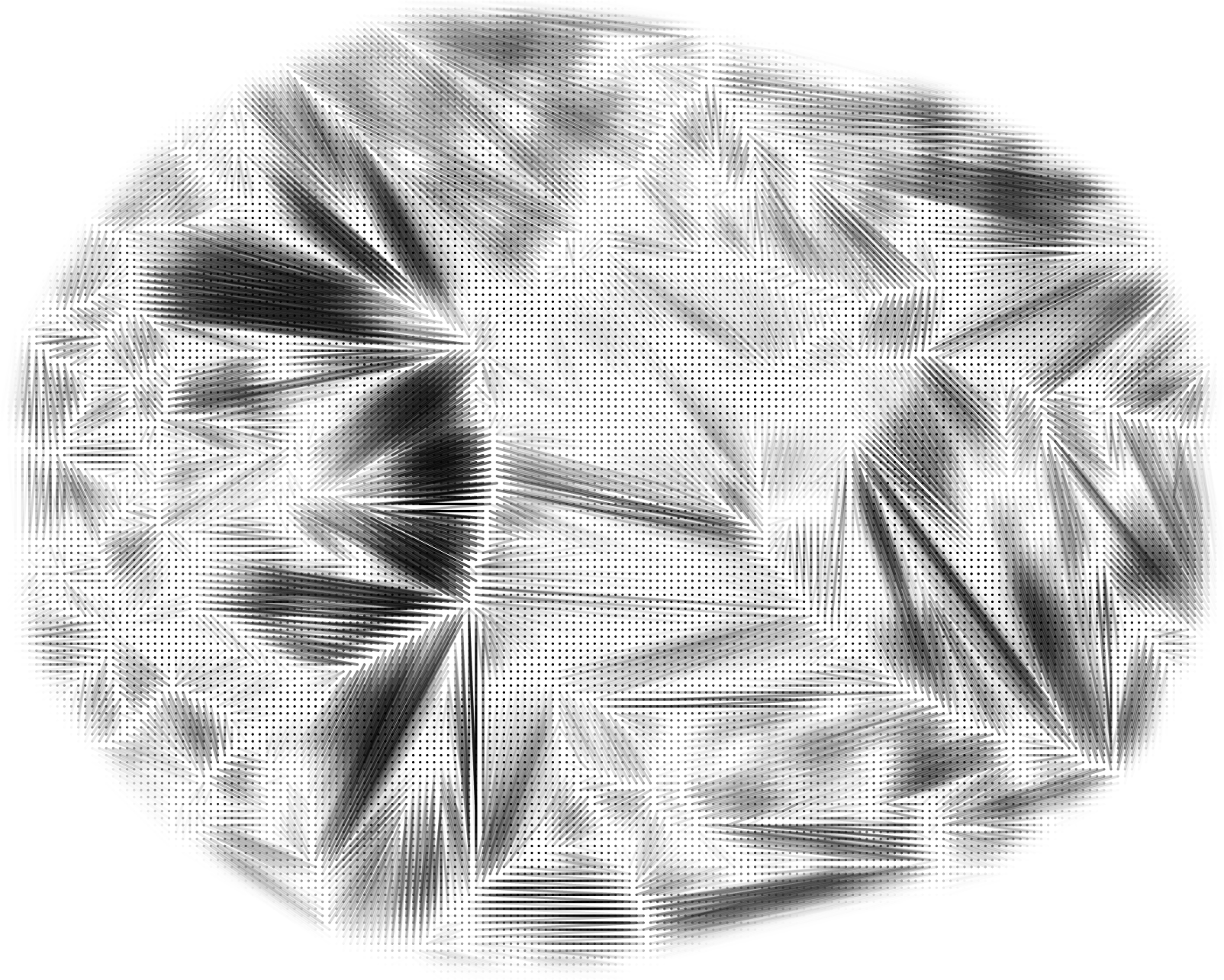} &
\includegraphics[width=0.375\linewidth,angle=90]{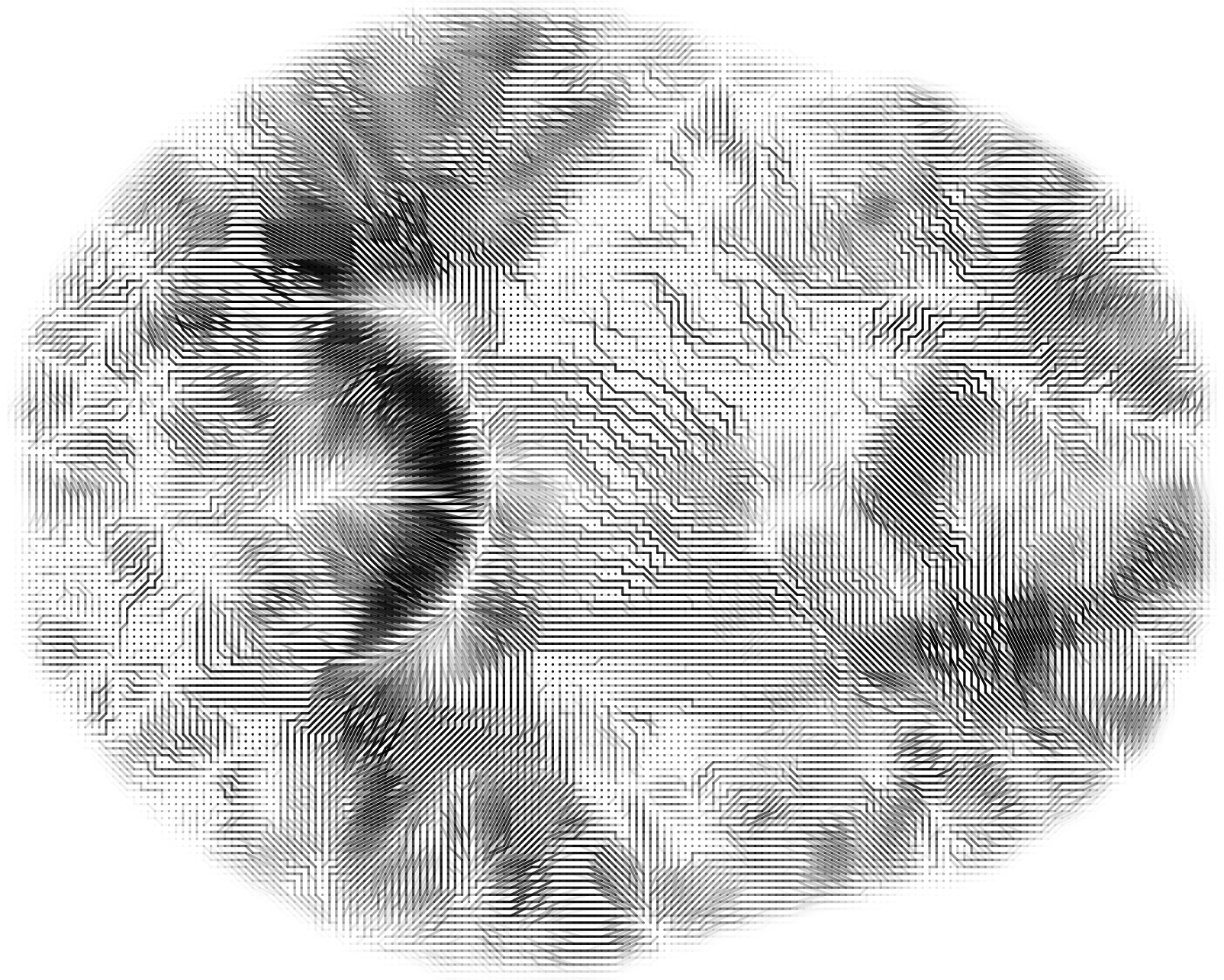} \\
                       \vspace{-0.1in} \\
                   (c) & (d)
\end{tabular}
  \vspace{-0.1in}
\caption{
  \label{fig:brains}
 Slice of (a) source and (b) target brain image and optimal transport map for
 (c) Euclidean and (d) squared Euclidean cost function. For the squared
 Euclidean cost the optimal transport solutions typically prefer to move many
 locations small distances, i.e. shift mass among nearest neighbors. For images
 the neighborhoods are on regular grids resulting in a staircase appearance of
 the transport plan.  } 
 \end{figure}

To compare the optimal transport distance to the LDDMM distance we compare how
well the distances can be used to predict clinical parameters. Using the pairwise
distances, we employ classical multidimensional scaling (MDS) to embed the pairwise
distances into Euclidean space. In this embedding each point corresponds to a 3D
brain image and thus provides a coordinate system for the relative locations
between the brain MRI's. The Euclidean structure of the embedding permits to
use standard statistical tools to form prediction models, in our case linear
regression.

Within this framework we compare Euclidean, LDDMM and optimal transport
distances as the input to the multidimensional scaling step.  Note, for the
Euclidean distance, i.e. treating each brain MRI as a point in Euclidean space,
classical multidimensional scaling is equivalent to a principal component
analysis.  
For the comparisons we used the OASIS brain database~\citep{oasis-brains}.  The
OASIS database consists of T1 weighted MRI of 416 subjects aged 18 to 96. One
hundred of the subjects over the age of 60 are diagnosed with mild to moderate
dementia. The images in the OASIS data set are already skull-stripped,
gain-field corrected and registered to the atlas space
of~\citet{talaraich:book88} with a 12-parameter affine transform. Associated
with the data are several clinical parameters. For the linear regression from
the MDS embeddings we restrict our attention to the prediction of age, mini
mental state examination (MMSE) and clinical dementia rating (CDR).

\begin{table}[tb]
\centering
\begin{footnotesize}
\begin{tabular}{l|c|c|c|c}
Model & 
Residual & $R^2$ &  $F$--statistic & $p$--value \\

\hline

${\rm age} =  a_0 + \sum_{i=1}^5 a_i l_i$ & 
{\bf 10.5} &  {\bf 0.82} & {\bf 404.9} & $ <\epsilon$ \\ 

${\rm age} =  a_0 + \sum_{i=1}^3 a_i x_i$  & 
10.87 &  0.82 & 639.5 & $ <\epsilon$ \\

${\rm age} =  a_0 + \sum_{i=1}^5 a_i z_i$  & 
12.04 &  0.78 & 297 & $ <\epsilon$ \\

${\rm age} = a_0 + a_1 z^1_1 + a_2 z^1_2 + a_3 z^2_1 + a_4 z^2_2 + a_5 z^3_1 +
a_6 z^4_2 + a_7 z^4_5 + a_8 z^5_1$  & 
10.9 & 0.82 & 239 & $<\epsilon$ \\

\hline

${\rm MMSE} = a_0 + a_1 {\rm age} $ & 
3.59 & 0.06 & 15.82 & 9.3e-05 \\

${\rm MMSE} = a_0 + a_1 l_1 $ & 
3.40 & 0.16 & 43.13 & 3.3e-10 \\

${\rm MMSE} = a_0 +a_1 x_1$ & 
3.36 & 0.18 &  50.30 & 1.6e-11 \\
 
${\rm MMSE} = a_0 + a_1 z_1 + a_2 z_3 + a_3 z_5$ & 
3.38 & 0.17 &  16.31 & 1.2e-09 \\

${\rm MMSE} = a_0 + a_1 z^2_2 + a_2 z^3_2 + a_3 z^3_4 + a_4 z^4_5 +
a_5 z^5_1 + a_6 z^5_3 + a_7 z^6_3$ & 
{\bf 3.14 } & {\bf 0.30 } & {\bf 14.03} & {\bf 4.3e-15 }\\

\hline

${\rm CDR} = a_0 + a_1 {\rm age} $ & 
0.27 & 0.25 & 144.5 & $ <\epsilon$ \\

${\rm CDR} = a_0 + a_1 l_1$ & 
0.26 & 0.34 & 223.9 & $ <\epsilon$ \\ 

${\rm CDR} = a_0 + a_1 x_1$ & 
0.25 & 0.36 & 248.5 & $ <\epsilon$ \\ 
 
${\rm CDR} = a_0 + a_1 z_1 + a_2 z_3 + a_3 z_5$ & 
0.26 & 0.35 & 77.5 & $ < \epsilon$ \\ 

${\rm CDR} = a_0 + a_1 z^3_4 + a_2 z^5_1 + a_3 z^5_3 + a_4 z^5_5 + a_5 a^6_3$ & 
{\bf 0.25} & {\bf 0.38} & {\bf 53.1} & $ < \epsilon$  \\ 

\end{tabular}
\end{footnotesize}
\caption{Optimal linear regression models from the OASIS data set, for  age,
  mini mental state examination (MMSE) and clinical dementia rating (CDR). The
  PCA coordinates from the Euclidean distances are denoted with $l_i$, the
  diffeomorphic manifold coordinates with $x_i$, the transport coordinates by
  $z_i$ and the multiscale transport coordinates with $z_i^j$ from coarsest
  $j=1$ to finest $j=6$ scale.  An entry with ``$ <\epsilon$'' denotes
  quantities smaller than machine precision. The best results are indicated in bold.  }
\label{tab:regression_oasis}
\end{table}
The MDS computation requires pairwise distances.  To speed up computations we
reduce the number of points by downsampling the volumes to size $44 \times 52
\times 44$. This results, discarding zero intensity voxels, in point clouds of
approximately $40,000$ points for each brain MRI. A single distance computation
with capacity propagation takes on the order of 10 seconds resulting in a total
computation time of around 2 weeks for all pairwise distances. For embedding
the optimal transport distance we consider two approaches:
\begin{compactenum}
\item A five dimensional MDS embedding based on the transport cost at the finest
  scale.
\item A multiscale embedding using multiple five dimensional MDS embeddings, one
  for the transport cost at each scale. 
\end{compactenum}

From the embeddings we build linear regression models, using the embedding
coordinates as independent variables and the clinical variables as dependent
variables.  As in~\citet{gerber:media10} we use the Bayesian information
criterion (BIC) on all regression subsets to extract a models that trade-off
complexity, i.e. number of independent variables, with the quality of fit.
Table~\ref{tab:regression_oasis} shows the results of the optimal transport,
with squared Euclidean cost, distances compared to the results reported
in~\citet{gerber:media10}. The transport based approach shows some interesting
behaviours. The single scale model performs worse on age while performing
similar on MMSE and CDR. The multiscale transport models perform similar to the
LDDMM approach except for MMSE where it almost doubles the explained variance
$R^2$.  This suggests that the multiscale approach captures information about
MMSE not contained in a single scale and prompts further research of multiscale
based models to predict clinical parameters.

\section{Conclusion}
The multiscale framework extends the size of transport problems that can be
solved by linear programming by one to two orders of magnitude in the size of
the point sets. 

The framework is flexible, the linear program at each scale can be solved by
any method. Depending on the refinement strategy a dual solution would need to
be constructed as well. The method can also be applied to solve the linear
assignment problem. The solution at the finest scale will be binary if $n=m$,
however, at intermediate steps a binary solution is not guaranteed unless at
each scale the source and target sets have the same number of points. This
could be enforced during the construction of the multiscale structures. As
currently defined the method requires point set inputs for the multiscale
constructing of the transport problems. However, one could design methods that
construct such a multiscale structure from a cost matrix only. However, for
large problems pairwise computations of all costs is typically prohibitively
expensive.

The multiscale transport framework provides several options for further
research.  The framework can be combined with various regularizations, e.g.
depending on scale and location and induces a natural multiscale decomposition
of the transport map. This decomposition can be used to extract information
about relevant scales in a wide variety of application domains. We are
currently investigating the use of such multiscale decompositions for
representation and analysis of sets of brain MRI.

The capacity constraint propagation approach suggests a different venue for
further exploration. The capacity propagation strategy does not hinge on a
geometrical notion of neighborhoods and is a suitable candidate to extend the
multiscale framework to non-geometric problems. Studying the interdependency
between the cost function, hierarchical decomposition of the transport problem
(or possibly more generic linear programs) and the efficiency of the capacity
constraint propagation is a challenging but interesting problem.

\acks{"We thank the referees for their constructive feedback. This work was
supported, in part, by NIH/NIBIB and NIH/NIGMS via 1R01EB021396-01A1:
Slicer+PLUS: Point-of-Care Ultrasound." }

\bibliography{MyPublications,sgerber,optimal-transport,multiscale,DiffusionBib}

\end{document}